\documentclass[hidelinks,onefignum,onetabnum]{siamart251216}

\usepackage{lipsum}
\usepackage{amsfonts}
\usepackage{graphicx}
\usepackage{epstopdf}
\usepackage{algorithmic}
\ifpdf
  \DeclareGraphicsExtensions{.eps,.pdf,.png,.jpg}
\else
  \DeclareGraphicsExtensions{.eps}
\fi

\newsiamremark{remark}{Remark}
\newsiamremark{hypothesis}{Hypothesis}
\crefname{hypothesis}{Hypothesis}{Hypotheses}
\newsiamthm{claim}{Claim}
\newsiamremark{fact}{Fact}
\crefname{fact}{Fact}{Facts}

\headers{Gaussian Transformer Dynamics}{A. Alcalde, Z. Ji, and E. Zuazua}

\title{Reachability and Asymptotics of Gaussian Transformer Dynamics\thanks{Submitted to the editors DATE.
\funding{AA was funded by the European Union's Horizon Europe MSCA project ModConFlex (grant number 101073558). EZ was funded by the Alexander von Humboldt-Professorship program, the ERC Advanced Grant CoDeFeL, the Grants PID2020-112617GB-C22 KiLearn and TED2021-131390B-I00-DasEl of MINECO and PID2023-146872OB-I00-DyCMaMod of MICIU (Spain),  the European Union's Horizon Europe MSCA project ModConFlex (grant number 101073558), the Transregio 154 Project ``Mathematical Modelling, Simulation and Optimization Using the Example of Gas Networks'' of the DFG, the AFOSR 24IOE027 project, and the SURE-AI Centre grant 357482, Research Council of Norway.}}}

\author{Albert Alcalde\thanks{Chair for Dynamics, Control, Machine Learning \& Numerics (Alexander von Humboldt Professorship), Department of Mathematics, Friedrich--Alexander-Universit\"at Erlangen--N\"urnberg, 91058 Erlangen, Germany. (\email{albert.alcalde@fau.de}, \email{zhengping.ji@fau.de}, \email{enrique.zuazua@fau.de})}
\and Zhengping Ji\footnotemark[2]
\and Enrique Zuazua\footnotemark[2]
\thanks{Departamento de Matem\'{a}ticas,
Universidad Aut\'{o}noma de Madrid,
28049 Madrid, Spain.
\\
\hspace*{\parindent}\phantom{\textsuperscript{$\dagger$}} Chair of Computational Mathematics, Fundaci\'{o}n Deusto. Av. de las Universidades, 24, 48007 Bilbao, Basque Country, Spain.
}
}

\usepackage{amsopn}

\usepackage{subcaption}

\usepackage{booktabs}
\usepackage{tabularx}
\usepackage{mathtools}
\usepackage{makecell}

\newcommand{\R}{\mathbb{R}}

\newcommand{\E}{\mathbb{E}}

\newcommand{\dd}{\mathrm{d}}
\DeclareMathOperator{\spec}{spec}
\DeclareMathOperator{\tr}{tr}
\DeclareMathOperator{\rank}{rank}

\ifpdf
\hypersetup{
  pdftitle={Reachability and Asymptotics of Gaussian Transformer Dynamics},
  pdfauthor={A. Alcalde, Z. Ji, and E. Zuazua}
}
\fi

\begin{document}

\maketitle

% REQUIRED
\begin{abstract}
We formulate data propagation through the Transformer, the machine learning architecture powering large language models, as a nonlinear control system on the space of probability measures. For the mean-field Transformer model with self-attention and affine feed-forward layers, we prove that Gaussian distributions remain exactly Gaussian along the induced flow. This invariance reduces the infinite-dimensional measure dynamics to a finite-dimensional bilinear control system governing the evolution of the mean and covariance, reformulates the expressive capacity of Transformers as a reachability problem for prescribed Gaussian moments, and reveals a novel connection with Riccati-type equations from classical filtering and control.

For time-varying controls, we prove exact finite-time reachability of any target Gaussian distribution whose covariance matrix has the same rank as the initial one, this rank constraint being an intrinsic invariant of the dynamics. For time-invariant parameters, we derive explicit spectral conditions leading either to asymptotic stability toward positive-definite equilibria or to finite-time blow-up of the covariance.

Numerical experiments complement the theory by showing that practical Transformers with Gaussian inputs remain close to moment-matched Gaussian distributions through early and intermediate layers, while Transformers with prescribed attention matrices reproduce the predicted covariance regimes: bounded evolution in stabilizing configurations and blow-up in destabilizing ones.

\end{abstract}

% REQUIRED

% REQUIRED

\section{Introduction}

Transformers \cite{bahdanau2014neural, vaswani2017attention} have become the dominant architecture in modern machine learning, achieving state-of-the-art performance in natural language processing \cite{achiam2023gpt4, devlin2019bert}, computer vision \cite{carion2020end, liu2021swin}, genomics \cite{abramson2024accurate, jumper2021highly}, and scientific machine learning \cite{bodnar2025foundation, price2025probabilistic}. Despite their empirical success, a rigorous mathematical framework characterizing exactly when and how Transformers reliably represent and propagate information remains elusive, which has motivated a growing theoretical effort to understand them through continuous-depth and mean-field limits, using tools from nonlinear control \cite{bruno2025emergence, burger2025analysis, castin2025unified, geshkovski2025mathematical, sander2022sinkformers}. When the number of layers and inputs becomes large, the Transformer model can be viewed as a controlled nonlinear flow acting on probability measures~\cite{peyre2025optimal}, in which the controlled state space is the density of the distribution of the inputs, the layer index is interpreted as a continuous time variable, and the layer-varying parameters of the Transformer serve as controls. 

A natural and analytically tractable setting for studying this measure flow is the invariant manifold of Gaussian input distributions \cite{castin2025unified}. 
This setting is widely adopted in modeling independently sampled data (see \Cref{ss:motivation}) and, crucially, the Gaussian invariant manifold reduces the infinite-dimensional controlled evolution characterizing the information-propagation properties of Transformers into a finite-dimensional control system for the mean and covariance (see \Cref{ss:GaussianTransformer}). The Gaussian framework provides a rigorous control-theoretic paradigm to investigate the training dynamics and expressivity of Transformers. Specifically, the approximation capacity of a Transformer translates directly into a reachability problem: does there exist a sequence of time-varying controls (weight parameters) capable of steering an initial Gaussian measure along the nonlinear flow to match an arbitrary target Gaussian? Furthermore, understanding the forward-pass dynamics naturally raises fundamental questions of asymptotic behavior: under what control parameter regimes does the Transformer flow stabilize to a stationary distribution, and when does it diverge? In the language of control theory, these properties correspond precisely to the system's controllability and stability \cite{tabuada2022universal}.

In this paper, we address the key questions of reachability and asymptotic behavior of the nonlinear control system arising from the mean-field Transformer model in the Gaussian setting. The Transformer architecture poses distinct control-theoretic challenges: (i) due to the strong nonlinear nature of self-attention, the well-posedness of the mean-field evolution is not guaranteed when the initial measure is not of compact support; (ii) the Gaussian reduction leads to coupled nonlinear dynamics of the mean and covariance, complicating the use of classical geometric control techniques \cite{bianchini2003needle}; (iii) the non-commutativity of the control matrices arising from the self-attention mechanism presents significant obstacles to establishing the existence of equilibria, and the nonlinearities in the covariance flow complicate global stability analysis. We overcome these difficulties by resolvent techniques and perturbation analysis, explicitly bridging the  Transformer dynamics with classical Riccati theory.

\subsection{Our contributions}

The main contributions of this paper are as follows:
\begin{itemize}
    \item We study the mean-field Transformer model with affine feed-forward layers and self-attention, and prove that the class of Gaussian measures remains invariant under the resulting system. This yields a Riccati-type ODE system for the evolution of mean and covariance (\Cref{prop:gaussianTransformer}), resembling a classical formulation in optimal filtering and control theory: 
    \begin{align}\label{eq:GaussianTransPDEcompact}
        \begin{cases}
            \dot\mu    &= (A + V)\mu + V\Sigma B\mu + b, \\
            \dot\Sigma &= A\Sigma + \Sigma A^\top + V\Sigma B\Sigma + \Sigma B^\top \Sigma V^\top,
        \end{cases}
    \end{align}
    where $\mu(t) \in \R^d$ is the mean and $\Sigma(t)\in\R^{d\times d}$ is the covariance of the Gaussian at time $t$,  while $A(t), B(t), V(t) \in\R^{d\times d}$ and $b(t)\in\R^d$ are trainable parameters acting as controls. Further, when the feed-forward is constructed using ReLU activation, we derive quantitative estimates on the discrepancy between measure flows and the Gaussian evolution \eqref{eq:GaussianTransPDEcompact} (\Cref{prop:GaussianPreserv}). The argument is not specific to ReLU and indicates how analogous estimates may be obtained for more general Lipschitz activations.
    
    \item For time-varying parameter matrices, we show that the rank of the covariance matrix $\Sigma$ is preserved along the flow of \eqref{eq:GaussianTransPDEcompact} (\Cref{lem:RankPreserv}). Leveraging matrix congruence transformations, we construct explicit time-varying control paths that achieve exact finite-time reachability of any target Gaussian state sharing this initial covariance rank (\Cref{thm:finiteTimeReachability}).

    \item For time-invariant parameters, we characterize the long-time behavior of the mean/covariance dynamics \eqref{eq:GaussianTransPDEcompact}. For stabilizing parameter regimes, we show existence of positive-definite equilibria and derive sufficient conditions for local stability  (\Cref{thm:SigmaConvergenceVIdentity,thm:sigmaConvGeneralV}). Conversely, under destabilizing parameter configurations, the covariance blows up in finite time (\Cref{prop:SigmaBlowUp}). Moreover, one can asymptotically match an arbitrary target mean by  static choices of parameters under stabilizing conditions (\Cref{thm:meanMatching}). Interestingly, the resulting stability conditions are consistent with empirical observations in pretrained vision Transformers \cite{trockman2023mimetic}.

    \item We perform numerical experiments serving two complementary purposes beyond the exact Gaussian theory. First, we show that pretrained Transformers, although outside the assumptions of the affine mean-field model, preserve an approximately Gaussian moment structure over early and intermediate layers when initialized with Gaussian inputs. Second, we test the robustness of the covariance predictions in more realistic discrete architectures with nonlinear feed-forward blocks, showing bounded covariance growth in stabilizing sign configurations and blow-up in destabilizing ones.
\end{itemize}
Taken together, these results provide a rigorous framework for understanding data propagation in trained Transformers --- traditionally studied via static approximation theory --- from a dynamical and control-theoretic perspective.

\subsection{Motivating the Gaussian setting}\label{ss:motivation}  

The restriction to Gaussian input distributions is not merely an analytical convenience, but also a standard paradigm in theoretical studies of the in-context learning (ICL) capabilities of Transformers \cite{garg2022can,goel2026training,von2023transformers,zhang2024trained}. Broadly speaking, ICL investigates how Transformers can solve families of supervised learning tasks directly from contextual examples provided at inference time. In this setting, the input to the Transformer typically consists of a sequence $\{(x_1,y_1), (x_2,y_2), \dots, (x_m,y_m), (x_{\text{query}}, 0)\}$,
where each feature $x_i \in \mathbb{R}^d$ (including the query $x_{\text{query}}$) is independently sampled from a Gaussian distribution (often $x_i \sim \mathcal{N}(0,\Sigma)$), and the labels $y_i$ are generated by some unknown task-specific rule (for instance, $y_i = w^\top x_i$ in a linear regression task, with $w \sim \mathcal{N}(0,I_d)$). The Transformer is trained over many such tasks to predict the missing label $y_{\text{query}}$ from the contextual examples. An important aspect of this framework is that Gaussian inputs are not intended as realistic models of data distributions. Rather, they provide an analytically tractable ensemble allowing one to isolate and study the mechanisms by which attention layers aggregate and propagate information.

\subsection{Related work}

Our work builds on recent efforts to understand Transformer architectures through mean-field limits and controlled measure flows. We position our contribution with respect to three closely related directions: asymptotic analyses of attention dynamics, controllability of Transformers, and Riccati-type covariance dynamics.

\paragraph{Asymptotics and mean-field perspective on Transformer dynamics} 
A recurring theme in theoretical studies is that repeated application of attention layers can drive collapse phenomena: inputs cluster and attention matrices can become effectively low-rank, limiting expressivity. At the particle level, the asymptotic dynamics of Transformers only with self-attention layers have been studied in \cite{alcalde2025clustering,11494448, geshkovski2023emergence,pham2025dynamical}. 

These observations are further clarified in the mean-field regime, where the dynamics can be described by partial differential equations on probability measures in $\mathbb{R}^d$. Early contributions identify well-posedness \cite{sander2022sinkformers}, clustering \cite{burger2025analysis,geshkovski2025mathematical}, metastability \cite{alcalde2026quantifying,bruno2025emergence,bruno2025a,geshkovski2024dynamic}, and phase transitions relevant to long-context attention \cite{chen2025critical}. Closest to our work is \cite{castin2025unified}, which proves Gaussian invariance for self-attention-only Transformers and derives corresponding mean and covariance equations. Their asymptotic results rely on a strong commutativity condition between the control matrices and the initial covariance. Our contribution is to move beyond the self-attention-only setting by incorporating affine feed-forward layers, and to analyze the resulting Gaussian system relying on weaker sign assumptions on the control matrices.

\paragraph{Reachability and simultaneous controllability of Transformers} The question of controllability and target reachability is central to Transformers. At the discrete level, a first approximate simultaneous controllability result is proved in \cite{yun2019transformers}, and extended to the exact setting in \cite{alcalde2025exact,kim2023provable}. From a mean-field perspective, the simultaneous controllability of Transformers including normalization layers has been established in \cite{geshkovski2024measure}, while \cite{akman2026optimal} takes an optimal control approach to study the training dynamics of Transformers. Although related, our work takes the different perspective of studying how a target Gaussian can be reached with minimal assumptions on the controls.

\paragraph{Bilinear systems and covariance control} As established in \Cref{prop:gaussianTransformer}, \eqref{eq:GaussianTransPDEcompact} forms a bilinear system, revealing a surprising link between Transformer models and the linear quadratic regulator theory, as the covariance matrix $\Sigma$ can be viewed as a ``feedback gain'' of the mean $\mu$ (see \Cref{ss:VetaId}). It is also related to the optimal estimation problem which aims at designing feedback controllers for the closed-loop system to reach a specified state covariance \cite{hotz1987covariance}. Unlike in optimal control where the Riccati equation is a tool to find optimal gains with guaranteed well-posedness, in our analysis the Riccati/Bernoulli-type equation is part of the state dynamics, the stability of which is under question: the coupled structure \eqref{eq:GaussianTransPDEcompact} is hence fundamentally new due to the specific structure of Transformers, whose behavior is complicated by the inherent lack of commutativity between control parameters and the covariance. Our controllability analysis of $\mu$ and $\Sigma$ based on congruence transformations  avoids the complexity of bracket computation in traditional nonlinear control theory \cite{elliott2009bilinear}.

\subsection{Organization of the paper} 

The remainder of the paper is organized as follows. In \Cref{sec:gaussianTransformer}, we introduce the Gaussian Transformer model. \Cref{sec:finiteTime} is devoted to finite-time reachability of the model, while \Cref{sec:asymptoticDynamics} addresses its asymptotic behavior for time-invariant parameters. In \Cref{s:numerics}, we present our numerical experiments and \Cref{sec:conclusions} concludes the paper, identifying future perspectives.

\section{Transformer dynamics with Gaussian initial conditions}\label{sec:gaussianTransformer} 

We study a class of nonlinear transport equations arising as mean-field limits of Transformer architectures \cite{vaswani2017attention}. Following recent developments at the interface of machine learning, optimal transport, and control theory \cite{castin2025unified, geshkovski2024measure,sander2022sinkformers}, these models describe the evolution of probability measures driven by parameterized, nonlocal vector fields encoding the architecture of the network. We briefly recall the modeling pathway leading to the equations considered in this work, referring to \cite{geshkovski2025mathematical} for detailed derivations from the discrete architecture.

The Transformer is a deep neural network model operating on a finite collection of particles
$\{x_i^\ell\}_{i=1}^n \subset \mathbb R^d$, called \textit{tokens}, which represent, for instance, words in a sentence or pixels in a picture. 
In the infinite-depth limit, their evolution is modeled as a continuous-time interacting particle system governed by
\begin{equation}\label{eq:particle_ode}
\dot x_i(t)
= \sigma(A(t)x_i(t)+b(t))
+ \sum_{j=1}^n
\frac{e^{x_j(t)^\top B(t) x_i(t)}}{\sum_{\ell=1}^n e^{x_\ell(t)^\top B(t) x_i(t)}}
\, V(t)x_j(t).
\end{equation}
In this system, \(\sigma:\mathbb{R}^d \to \mathbb{R}^d\) denotes a nonlinearity applied componentwise known as the \emph{activation function}, while \(A(t), B(t), V(t) \in \mathbb{R}^{d\times d}\) and \(b(t)\in\mathbb{R}^d\) are time-dependent parameters inherited from the trained network. 
Motivated by the increasing context lengths of modern Transformers \cite{press2022train}, we are interested in studying the limit as the number of particles $n\to\infty$. Introducing the empirical measure $\rho_t^n \coloneqq \frac{1}{n}\sum_{i=1}^n \delta_{x_i(t)}$, the right-hand side of \eqref{eq:particle_ode} can be written as a functional of $\rho_t^n$.
As $n\to\infty$, for compactly supported input measures, this system converges (in the sense made precise in \cite{castin2025unified}) to a well-posed continuity equation
\begin{equation}\label{eq:transformerPDE}
\begin{cases}
\partial_t \rho_t + \nabla_x\cdot\left(\rho_t \Gamma[\rho_t](t,x)\right) = 0,\\
\rho_{t=0} = \rho_0,
\end{cases}
\end{equation}
where $\rho_t \in \mathcal P(\mathbb R^d)$ and for any $\rho\in\mathcal{P}(\R^d)$, the velocity field $\Gamma[\rho]$ is defined as
\begin{equation}\label{eq:gammaRho}
\Gamma[\rho](t,x) = 
\sigma\left(A(t)x+b(t)\right)
+ 
\frac{\int e^{y^\top B(t) x} \, V(t) y \, \dd\rho(y)}
{\int e^{y^\top B(t) x} \, \dd\rho(y)}.
\end{equation}
We refer to \eqref{eq:transformerPDE} as the \emph{Transformer PDE}. The vector field \eqref{eq:gammaRho} naturally decomposes into a pointwise drift
$$
\mathcal F(t, x) \coloneqq \sigma(A(t)x + b(t)),
$$
corresponding to the so-called \emph{feed-forward} layers, and a nonlocal interaction term
\begin{equation}\label{eq:SAmechanism}
    \mathcal A[\rho](t, x)
    \coloneqq
    \frac{\int e^{y^\top B(t) x} \, V(t) y \, \dd\rho(y)}
    {\int e^{y^\top B(t) x} \, \dd\rho(y)},
\end{equation}
which corresponds to the \emph{self-attention} mechanism. From a control-theoretic viewpoint, \eqref{eq:transformerPDE} is a controlled continuity equation with measure-dependent drift, where the time-dependent matrices $A(t), B(t), V(t), b(t)$ act as control variables.

In this work, we focus on the dynamics of \eqref{eq:transformerPDE} for Gaussian input measures. 

As we will show below in \Cref{prop:gaussianTransformer}, when the initial measure $\rho_0$ is Gaussian and the activation function $\sigma$ is the identity, the solution $\rho_t$ remains Gaussian for all $t$, and the infinite-dimensional dynamics \eqref{eq:transformerPDE} reduce to an ODE system governing the evolution of the mean and covariance. We refer to the resulting reduced-order model as the
\emph{Gaussian Transformer}. Thus, the question of well-posedness of the model can be studied equivalently in this finite-dimensional setting. Alternatively, for the case of a Gaussian input measure and $\sigma$ being the ReLU activation function, we will show that the solution remains sub-Gaussian for short times (see \Cref{lem:subGaussian}). This allows us to extend the well-posedness guarantee to this setting.

Throughout the paper, we will denote a (symmetric) positive definite (respectively, positive semi-definite, negative definite and negative semi-definite) matrix $A \in \R^{d\times d}$ by $A\succ 0$ (resp. $A\succeq 0$, $A \prec 0$ and $A \preceq 0$).

\subsection{The Gaussian Transformer}\label{ss:GaussianTransformer}

For Gaussian measures, the self-attention operator admits an explicit linear representation. If $\rho = \mathcal N(\mu,\Sigma)$, then the \emph{attention-only} Transformer PDE \eqref{eq:transformerPDE} with $\sigma = 0$ induces the vector field
\begin{equation}\label{eq:SAonlyVectorField}
\mathcal A[\rho](t,x) = V(t)(\mu + \Sigma B(t) x)
\end{equation}
as shown in \cite{castin2025unified}. In particular, Gaussian measures are invariant under the self-attention-only flow, and their evolution reduces to a closed system of ODEs for $(\mu,\Sigma)$. We extend this structure to the full Transformer dynamics with an affine feed-forward term, i.e., $\sigma = \mathrm{id}$ in \eqref{eq:gammaRho}.

\begin{proposition}[Gaussian Transformer] \label{prop:gaussianTransformer}
Let $\rho_t$ be the solution of \eqref{eq:transformerPDE} with $\sigma= \mathrm{id}$ and initial condition
$\rho_0=\mathcal N(\mu_0,\Sigma_0)$, where $\Sigma_0\succeq0$. Then, there exists $T_{\max}>0$ such that for all $t \in [0, T_{\max})$, $\rho_t$ remains Gaussian, with its mean $\mu(t)$ and covariance $\Sigma(t)$ satisfying
\begin{equation}\label{eq:gaussianTransformer}
\begin{cases}
\dot\mu   &= \left(A(t) + V(t) +  V(t)\Sigma B(t) \right) \mu + b(t), \\
\dot\Sigma &= A(t)\Sigma + \Sigma A(t)^\top + V(t)\Sigma B(t)\Sigma + \Sigma B(t)^\top \Sigma V(t)^\top,
\end{cases}
\quad 
\begin{aligned}
    \mu(0) &= \mu_0, \\
    \Sigma(0) &= \Sigma_0.
\end{aligned}
\end{equation}
\end{proposition}
\begin{proof}
Here and below, unless otherwise stated, $A,B,V,b$ are evaluated at time $t$. By \eqref{eq:SAonlyVectorField}, the velocity field is given by
$$
\Gamma[\rho_t](t, x)
=A x+ b + V\mu + V\Sigma B x
=\big(A+V\Sigma B\big) x +\big(V\mu+b\big),
$$
which is affine in $x$. The pushforward of a Gaussian by an affine map is again Gaussian, so $\rho_t$ stays Gaussian. To derive dynamics for the mean and the covariance, we use standard moment identities: for any smooth $\phi:\R^d\to\R$, it holds that
\begin{equation}\label{eq:moment-identity}
\frac{\dd}{\dd t}\int \phi(x)\,\dd\rho_t(x)=\int \nabla\phi(x)\cdot \Gamma[\rho_t](t,x)\,\dd\rho_t(x).
\end{equation}
Next, we use the definition of $\mu(t)=\int x\,\dd\rho_t(x)$ and \eqref{eq:moment-identity} with $\phi(x)=x$ componentwise to obtain
\[
\dot\mu
=\int \Gamma[\rho_t](t, x) \,\dd\rho_t(x)
=\int \big(Ax+b+V\mu+V\Sigma Bx\big)\,\dd\rho_t(x)
= A\mu + b + V\mu + V\Sigma B \mu,
\]
which gives the desired mean dynamics. For the covariance dynamics, let $M(t):=\E_{\rho_t}[x x^\top]$, so $\Sigma=M-\mu\mu^\top$. Differentiating $M$ gives
\[
\dot M
=\E_{\rho_t}\big[x \Gamma[\rho_t](t,x)^\top+\Gamma[\rho_t](t,x) x^\top\big].
\]
Since $\Gamma[\rho_t](t,x)= A x +b+V\mu+V\Sigma Bx$, linearity and the identity
$\E_{\rho_t}[x]=\mu$
yield
\begin{align*}
\dot M
&=\E_{\rho_t}[x(Ax)^\top]+\E_{\rho_t}[x(V\Sigma Bx)^\top]
 + \E_{\rho_t}[xb^\top]+\E_{\rho_t}[x(V\mu)^\top] \\
&\quad+\E_{\rho_t}[(Ax)x^\top]+\E_{\rho_t}[(V\Sigma Bx)x^\top]
 + \E_{\rho_t}[b x^\top]+\E_{\rho_t}[(V\mu)x^\top] \\
&= M A^\top + M B^\top \Sigma V^\top + \mu b^\top + \mu\mu^\top V^\top
  + A M + V\Sigma B M + b\mu^\top + V\mu\mu^\top.
\end{align*}
Differentiating $\Sigma=M-\mu\mu^\top$ and substituting $\dot\mu$ gives
\begin{equation*}
\dot\Sigma
= \dot M - \dot\mu\,\mu^\top - \mu\,\dot\mu^\top =  A\Sigma+\Sigma A^\top + V\Sigma B\Sigma + \Sigma B^\top \Sigma V^\top,
\end{equation*}
which completes the proof.
\end{proof}

\subsection{Transformer dynamics with ReLU feed-forward layers}

\Cref{prop:gaussianTransformer} shows that Gaussian distributions are invariant under the Transformer PDE \eqref{eq:transformerPDE} when the activation function in the feed-forward layers is set to the identity. In this section, we consider the dynamics of the distribution $\rho_t$ in the presence of ReLU activation 
$\sigma(x) = \max(0, x)$ (componentwise), and quantitatively show that it remains close to the Gaussian evolution \eqref{eq:gaussianTransformer}, hence justifying the well-posedness of the Gaussian Transformer under ReLU activations.

First, we prove the following result about the sub-Gaussian behavior of the dynamics \eqref{eq:transformerPDE} for short times. Throughout, $\lambda_{\max} (M)$ (resp. $\lambda_{\min} (M)$) denotes the largest (resp. smallest) eigenvalue of a matrix $M$. 
\begin{lemma}\label{lem:subGaussian}
     Let $A$, $b$, $B$, $V$ be fixed parameters. Denote by ${\rho}_t$ the solution of \eqref{eq:transformerPDE} with \emph{ReLU} activation $\sigma = \max(0, x)$ and initial condition ${\rho}_0=\mathcal N(\mu_0,\Sigma_0)$.  Assume $\kappa_0\in\R$ satisfies 
     $4\|B\| \le \kappa_0 < \frac{1}{2\lambda_{\max}(\Sigma_0)}$ and define
    $$
    E_t := \int_{\mathbb{R}^d} e^{\kappa_0 |x|^2} \dd\rho_t(x).
    $$
    Then, $\rho_t$ is sub-Gaussian in short time, i.e., there exists $T^*>0$ such that $E_t \le 2E_0$ for all $t\in[0,T^*]$.
\end{lemma}
The proof is relegated to \Cref{app:lemProof}. Thanks to \Cref{lem:subGaussian}, we have the following quantitative estimate on the short-time preservation of Gaussianity in \eqref{eq:transformerPDE} in the presence of ReLU activation.

\begin{proposition}\label{prop:GaussianPreserv}
Let $A$, $B$, $V$ and $b$ be fixed parameters. Denote by ${\rho}_t$ the solution of \eqref{eq:transformerPDE} with \emph{ReLU} activation $\sigma = \max(0, x)$ and initial condition ${\rho}_0=\mathcal N(\mu_0,\Sigma_0)$, where $\Sigma_0\succeq0$. If $\lambda_{\max}(\Sigma_0)<\frac{1}{8\|B\|}$, then for $t\in[0,T^*]$,
    $$
    W_2(\rho_t, \nu_t) \le t \cdot \left\| \max(0, -(Ax + b)) \right\|_{L^2(\rho_0)} + \mathcal{O}(t^2),
    $$ 
where $W_2(\cdot,\cdot)$ is the Wasserstein-2 distance between distributions, $\nu_t$ is the Gaussian evolution solving \eqref{eq:transformerPDE} with $\sigma=\mathrm{id}$ and $\nu_0=\rho_0$.
\end{proposition}

The proof is postponed to \Cref{app:proProof}. We nonetheless note that the argument for the short-time preservation mechanism in \Cref{prop:GaussianPreserv} is not specific to ReLU. More generally, for other Lipschitz nonlinear activations $\sigma$, under suitable assumptions on the initial covariance matrix, one may expect the existence of a time $T_\sigma>0$ such that, for $t\leq T_\sigma$, $W_2(\rho_t, \nu_t) \le t \cdot \left\| \sigma(AX_0+b) - (AX_0 + b)\right\|_{L^2(\rho_0)} + \mathcal{O}(t^2)$. Thus, if $\sigma(AX_0+b)$ is close to $AX_0+b$ under the initial law, the nonlinear evolution remains close to the corresponding Gaussian evolution for short times.

\Cref{prop:GaussianPreserv} implies that the Gaussian model \eqref{eq:gaussianTransformer} provides a quantifiable short-time approximation of the nonlinear Transformer PDE \eqref{eq:transformerPDE}. Although Gaussianity is not preserved for nonlinear activations, the result quantifies a tube of validity around the Gaussian dynamics. This justifies focusing the subsequent analysis on the Gaussian setting, where the evolution is finite-dimensional, guaranteeing that the analysis on the Gaussian Transformer \eqref{eq:gaussianTransformer} is informative for real models in practice (see \Cref{s:numerics} for empirical validation with pretrained Transformer models).

\section{Finite-time reachability}\label{sec:finiteTime}

Our first set of results concerns finite-time behavior and the reachability set of the Gaussian Transformer \eqref{eq:gaussianTransformer} under all possible choices of the time-varying parameters $A(t)$, $V(t)$ and $B(t)$. 

The following result establishes that the rank of $\Sigma(t)$ along the flow \eqref{eq:gaussianTransformer} cannot change as a function of $t$, so it is a natural invariant of the flow \eqref{eq:gaussianTransformer}.
\begin{lemma}[Rank preservation]\label{lem:RankPreserv}
    Let $A, B, V$: $[0,\infty)\rightarrow\R^{d\times d}$ be time-varying matrices. For any finite time interval $[0,T_{\max})$ on which the solution of \eqref{eq:gaussianTransformer} exists, we have
    $\rank(\Sigma(t)) = \rank(\Sigma_0)$, $\forall t\in [0,T_{\max})$.
\end{lemma}
\begin{proof}
We define a time-varying matrix $M(t) \coloneqq A(t) + V(t)\Sigma(t)B(t)$ for $t\in[0,T_{\max})$. By symmetry of the equation we have
$\dot{\Sigma}(t) = M(t)\Sigma(t) + \Sigma(t)M(t)^\top$. This is a linear time-varying equation in terms of a congruence transformation. Let $\Psi(t, 0)$ be the state transition matrix from time $0$ to time $t$ for the linear system $\dot{x}(t) = M(t)x(t)$, meaning that it is the unique solution of
$$
\frac{\dd}{\dd t}\Psi(t, 0) = M(t)\Psi(t, 0), \quad \Psi(0, 0) = I.
$$
Then the unique solution for $\Sigma(t)$ is
\begin{align}\label{4.2}
    \Sigma(t) = \Psi(t, 0) \Sigma(0) \Psi(t, 0)^\top.
\end{align}
The matrix $\Psi(t, 0)$ as the state transition matrix of the equation $\dot{x}(t) = M(t)x(t)$ is always invertible for any time $t\in[0,T_{\max})$. This is because its determinant is given by Liouville's formula (see, for example, \cite[Lemma 3.11]{Teschl2012ODE}) as
$$
\det(\Psi(t, 0)) = \exp\left( \int_0^t \tr(M(\tau)) \dd\tau \right) \neq 0.
$$
Therefore, \eqref{4.2} and the non-singularity of $\Psi(t,0)$ imply that the rank of $\Sigma(t)$ in \eqref{eq:gaussianTransformer} is invariant and equal to the rank of $\Sigma(0)$ for all $t$ where the solution exists. 
\end{proof}
\begin{theorem}\label{thm:finiteTimeReachability}
Suppose $\Sigma_0\succeq0$. For any $T>0$, $\hat{\mu}\in\R^d$ and $\hat{\Sigma}\succeq0$ satisfying $\rank(\hat{\Sigma})=\rank(\Sigma_0)$, there exists continuous $A$, $B$, $V:$ $[0,T]\rightarrow\R^{d\times d}$ and $b: [0,T]\rightarrow\R^d$ such that $\mu(t)$, $\Sigma(t)$ solving \eqref{eq:gaussianTransformer} satisfy $\mu(T)=\hat{\mu}$, $\Sigma(T)=\hat{\Sigma}$.
\end{theorem}
\begin{proof}
Let $A(t)=-V(t)-V(t)\Sigma(t)B(t)$,  and $B(t)=\Sigma(t)^{\dagger}$ be the Moore--Penrose pseudoinverse of $\Sigma(t)$, acting as smooth feedback controls. Then the dynamics of $\mu$ and $\Sigma$ become
\begin{equation}\label{eq:SigmaLinearDynamics}
    \begin{cases}
        \dot\mu(t) = b(t), \\
        \dot\Sigma(t) = -V(t)\Sigma(t) -\Sigma(t) V(t)^\top, 
    \end{cases}
\quad 
\begin{aligned}
    \mu(0) &= \mu_0, \\
    \Sigma(0) &= \Sigma_0.
\end{aligned}
\end{equation}
We shall first prove that there exists $V:[0,T]\rightarrow\R^{d\times d}$ such that $\Sigma(t)$ satisfies $\Sigma(0)=\Sigma_0$ and $\Sigma(T)=\hat{\Sigma}$. By the same argument as in \Cref{lem:RankPreserv}, the solution of \eqref{eq:gaussianTransformer} becomes $\Sigma(t) = \Psi_V(t) \Sigma_0 \Psi_V(t)^\top$ where the transition matrix $\Psi_V(t)$ is the unique solution of 
\begin{align}\label{4.3}
    \dot{\Psi}_V(t) = -V(t)\Psi_V(t), \quad \Psi_V(0) = I_d.
\end{align}
Thus, it suffices to show there exists $V(t)$ such that $\hat{\Sigma}=\Sigma(T)=\Psi_V(T)\Sigma_0\Psi_V(T)^\top$. 

We start by constructing $\Psi_V(T)$. Since $\Sigma_0\succeq0$ and $\hat{\Sigma}\succeq0$, we can write their spectral decompositions as 
$\Sigma_0 = U_0 \Lambda_0 U_0^\top$, $\hat{\Sigma} = U_1 \Lambda_1 U_1^\top$,
where $U_0, U_1 \in \mathrm{SO}(d)$. Without loss of generality, assume that $\rank(\Sigma_0)=\rank(\hat{\Sigma})=r$, and that the eigenvalues are arranged so that the positive entries occupy the top-left $r \times r$ block of the diagonal matrices $\Lambda_0$, $\Lambda_1$, while the remaining blocks are zero.

Let $\lambda_{0,i}$ and $\lambda_{1,i}$ be the $i$-th diagonal elements of $\Lambda_0$ and $\Lambda_1$, $i=1,\ldots,d$. Construct a nonsingular diagonal matrix $D$ such that
$D_{ii} = \sqrt{\tfrac{\lambda_{1,i}}{\lambda_{0,i}}}$ for $1 \leq i \leq r$, and 
        $D_{ii} =1$ for $r < i \leq d$.
        
Let $M=U_1DU_0^\top$. Then $\det(M)>0$. Since $U_0$, $U_1$ are orthogonal matrices, it is easy to verify that $M\Sigma_0 M^\top = U_1 \Lambda_1 U_1^\top = \hat{\Sigma}$, and hence $M$ is the matrix $\Psi_V(T)$ required. Next we construct the time-varying matrix $V$.

Since $\det(M)>0$ and the space of $d$-dimensional nonsingular real matrices with strictly positive determinants is path-connected, there exists a differentiable path $\Phi:[0,T]\rightarrow \mathrm{GL}^+(d,\R)$ such that $\Phi(0)=I_d$, $\Phi(T)=M$. 

Let $V(t):=-\dot{\Phi}(t)\Phi(t)^{-1}$, then $V(t)$ is the time-varying matrix allowing the matrix $\Psi_V(t)$ in \eqref{4.3} to satisfy $\hat{\Sigma}=\Sigma(T)=\Psi_V(T)\Sigma_0\Psi_V(T)^\top$. 

Meanwhile, by choosing $b(t)=\frac{1}{T}(\hat{\mu}-\mu_0)$, one obtains the finite-time reachability of $\mu(t)$ in \eqref{eq:SigmaLinearDynamics}. This proves the result.
\end{proof}
\begin{remark}
\cref{thm:finiteTimeReachability} ensures the existence of time-varying controls (i.e. layer-varying parameters) to match an arbitrary initial Gaussian with rank-compatible target Gaussians, guaranteeing the approximation capacity of Gaussian Transformers. However, from classical control theory one knows that the time-varying control parameters might be physically infeasible when the target lies near the boundary of reachability sets \cite{chen2015optimal}, and hence autonomous choice of parameters across layers might be preferred in control applications.
\end{remark}

\section{Asymptotic dynamics}\label{sec:asymptoticDynamics}

\begin{table}[t]
\caption{Asymptotics of \eqref{eq:gaussianTransformer} with $\Sigma(0)=\Sigma_0\succ0$. The first row follows from \Cref{thm:SigmaConvergenceVIdentity,thm:meanMatching}, 
the second from \Cref{thm:sigmaConvGeneralV}, and the third from \Cref{prop:SigmaBlowUp}.}
\label{tab:asymptoticSummary}
\centering
\small
\setlength{\tabcolsep}{5pt}
\renewcommand{\arraystretch}{1.08}
\begin{tabularx}{\textwidth}{@{}p{0.2\textwidth}@{\hspace{0pt}}p{0.2\textwidth} X X@{}}
\hline
\textbf{$V, B$ regime}
& \textbf{Extra condition}
& \textbf{Covariance $\Sigma(t)$}
& \textbf{Mean $\mu(t)$} \\
\hline
\makecell[l]{$V=\eta I_d$,\\$\eta(B+B^\top)\preceq0$}
&
$\eta<0$
&
$\Sigma(t)\to
U\left(\begin{smallmatrix}
\Sigma^a_\infty & 0\\
0 & 0
\end{smallmatrix}\right)U^\top$
&
$\mu(t) \to \mu_\infty \in \R^d$
\\
&
$\eta>0$
&
$\Sigma(t)\to
U\left(\begin{smallmatrix}
\Sigma^a_\infty & 0\\
0 & 0
\end{smallmatrix}\right)U^\top$
&
$\|\mu(t)\|\to\infty$
\\[1mm]
\hline
$V\prec0$, $B\succ0$
&
$\Re(\spec(A))>0$
&
$\Sigma(t)\to\Sigma_\infty\succ0$
&
Depends on $\spec (A + V + V\Sigma_\infty B)$
\\
&
$\Re(\spec(A))\leq0$
&
$\Sigma(t)\to0$
&
Depends on $\spec(A+V)$
\\[1mm]
\hline
$V\succ0$, $B\succ0$
&
$\lambda_{\min}(A+A^\top)\ge0$
&
$\|\Sigma(t)\|\to\infty$
&
$\|\mu(t)\|\to\infty$
\\
&
$\lambda_{\min}(A+A^\top)<0$
&
Depends on $\Sigma_0$
&
Depends on $\Sigma_0$
\\
\hline
\end{tabularx}
\end{table}

To characterize the asymptotics of \eqref{eq:gaussianTransformer}, we  couple the system with suitable initial conditions $\mu(0) = \mu_0 \in \R^d$ and $\Sigma(0) = \Sigma_0 \succeq 0$, as well as restrict the study to constant-in-time parameters $A, V, B \in \R^{d\times d}$ and bias $b \in\R^d$.

The asymptotic analysis of \eqref{eq:gaussianTransformer} is difficult due to the non-commutativity of $\Sigma$ with the parameter matrices. As the product of positive definite matrices is not necessarily positive definite, it is not easy to derive the existence of the equilibrium of the nonlinear matrix differential equation based on the positive-definiteness of $A$, $B$ and $V$. On the other hand, the asymptotics of the mean depend on the spectral properties of the limiting covariance matrix, and can be interpreted as the problem of characterizing stabilizing feedback gains for a linear system with static feedback.

Given these difficulties, the results in this section are split in two. In the first part, we consider the easier case of $V = \eta I_d$, $\eta \in \R$, in which the analysis simplifies thanks to the connections to Riccati theory, and weaker assumptions are needed on the other matrix parameters. In the second part, we tackle the more general yet involved case of definite $V\in \R^{d\times d}$. We provide a summary of results in \Cref{tab:asymptoticSummary}.

\subsection{Asymptotic dynamics for \texorpdfstring{$V=\eta I_d$}{V = eta Id}}
\label{ss:VetaId}

Before introducing our results regarding the asymptotic dynamics of \eqref{eq:gaussianTransformer} for $V = \eta I_d$, we make the connection to Riccati theory rigorous. This observation, although restricted to a certain parameter range, will prove to be useful in illustrating the dynamics of the Gaussian Transformer. 

Throughout this subsection, we will fix $V = \eta I_d$, $\eta \in \R$, and assume
\begin{equation}\label{eq:signOfetaB}
    \eta ( B + B^\top) \preceq 0.
\end{equation}
If \eqref{eq:signOfetaB} holds, we can write $- H^\top H \coloneqq \eta( B + B^\top)$ for some $H\in \R^{m \times d}$ and $m\leq d$. This yields
\begin{equation}\label{eq:DRE}
    \dot\Sigma = A\Sigma + \Sigma A^\top - \Sigma H^\top H\Sigma, \quad \Sigma(0) = \Sigma_0,
\end{equation}
which is a particular instance of the differential Riccati equation. For any initial condition $\Sigma_0 \succeq 0$, standard Riccati theory \cite[Theorem 4.1.6]{AbouKandil2003Riccati} proves the existence of a unique solution of \eqref{eq:DRE} for all times only for the case when $\eta (B + B^\top) \preceq 0$. Otherwise, finite-time blow-up of the solution may occur. 
In turn, the mean dynamics becomes 
\begin{align*}
    \dot \mu = (A + \eta I_d)\mu + \Sigma (\eta B) \mu  + b = (A - \Sigma H^\top H ) \mu + \eta (I_d - \Sigma B^\top ) \mu + b.
\end{align*}
The first term has the same structure as the Kalman--Bucy error dynamics with covariance \eqref{eq:DRE}, for which exponential decay is known under suitable assumptions \cite{ruymgaart2013mathematics}. However, the additional term $\eta(I_d-\Sigma B^\top)\mu$ has no direct analogue in the standard Kalman--Bucy setting and prevents us from inferring the asymptotic behavior of $\mu$ from Riccati theory alone.

Thus, in the case $V=\eta I_d$, the covariance dynamics admit a precise Riccati interpretation under \eqref{eq:signOfetaB}. This analogy is useful but incomplete: it does not by itself characterize the mean dynamics, and it does not extend directly to general value matrices. A separate analysis is therefore required.

In what follows, we will denote the spectrum of a matrix $A \in \R^{d\times d}$ by $\spec (A) \subset \mathbb{C}$, and by $\Re(\spec (A))$ the real part of its elements. The next theorem characterizes the asymptotic behavior of $\Sigma(t)$ in terms of the real part of the spectrum of $A$.

\begin{theorem}\label{thm:SigmaConvergenceVIdentity}
    Let $\Sigma(t)$ be the solution of \eqref{eq:DRE} with $\Sigma_0\succ0$. Suppose that $A$ has real Schur decomposition as
    \begin{align}\label{eq:ADecomp}    A=U\tilde{A}U^{\top}=U\begin{pmatrix}
        A_{a} & A_{12}\\
        0 & A_{n}
    \end{pmatrix} U^{\top}, 
    \end{align}
    where $U$ is an orthogonal matrix, $\Re(\spec(A_a)) > 0$ and $\Re(\spec(A_n)) \leq 0$. Define the negative semidefinite matrix $\tilde{B} \coloneqq \left(\begin{smallmatrix} B_{11} & B_{12} \\ B_{21} & B_{22} \end{smallmatrix}\right)= U^{\top} \eta (B+B^\top) U$, in which $B_{11}$ has the same dimension as $A_a$. Then when $t\rightarrow\infty$, we have 
    \begin{align*} 
        \Sigma(t)\rightarrow U\begin{pmatrix}
            \Sigma^a_{\infty} &0\\
            0 & 0
        \end{pmatrix}U^{\top}
    \end{align*}
    where $\Sigma^a_{\infty}$ is the unique positive definite solution of the algebraic Riccati equation 
    \begin{align}\label{eq:ARE}
        0=A_a\Sigma^a_{\infty}+\Sigma^a_{\infty} A_a^{\top}+\Sigma^a_{\infty}B_{11}\Sigma^a_{\infty}.
    \end{align}
\end{theorem}
\begin{proof}
    As $\Sigma_0$ is invertible, we consider the decomposition
    $\Sigma^{-1}(t) = U \tilde{P}(t) U^{\top}$ and study the dynamics of $\tilde{P}(t)$ for $t\geq 0$. Using that $U\frac{\dd \tilde{P}}{\dd t}U^{\top} =\frac{\dd \Sigma^{-1}}{\dd t}= -\Sigma^{-1} \dot{\Sigma} \Sigma^{-1}$, from \eqref{eq:DRE} we have
    $$
    \dot{\tilde{P}} = -\tilde{A}^\top \tilde{P} - \tilde{P} \tilde{A} - \tilde{B}.
    $$
    Writing $\tilde{P} = \begin{pmatrix} P_{11} & P_{12} \\ P_{21} & P_{22} \end{pmatrix}$ corresponding to the blocked form of $A$ and $\tilde{B}$, we have 
    \begin{align*}
        \begin{array}{l}
        \dot{P}_{11} = -A_a^\top P_{11} - P_{11} A_a  - B_{11}\\
        \dot{P}_{12} = -A_a^\top P_{12} - P_{11} A_{12} - P_{12} A_n - B_{12}\\
        \dot{P}_{21} = -A_{12}^\top P_{11} - A_n^\top P_{21} - P_{21} A_a - B_{21}\\
            \dot{P}_{22} = -A_n^\top P_{22} - P_{22} A_n - (A_{12}^\top P_{12} + P_{21} A_{12}) - B_{22}
        \end{array}
    \end{align*}
    Recalling the classical results of Lyapunov differential equations (for example, see \cite[Section 6.7, Theorem 7.5]{antsaklis2006linear}), one sees that $P_{11}$ will converge to the unique positive definite solution of $0=-A_a^\top P_{11} - P_{11} A_a  - B_{11}$, denoted by $P^{\infty}_{11}$, due to the fact that $-A_a$ is stabilizing.

    As for $P_{12}(t)$, since $P_{11}(t)$ remains bounded for all $t\geq 0$, the limit of $P_{12}(t)$ when $t\rightarrow\infty$ depends on the eigenvalues of $A_a$ and $A_n$, more specifically, its growth rate is
    \begin{align}\label{3.01}
        \lambda_{12} = \max \Re(\spec(-A_a^\top))  + \max \Re(\spec(-A_n)),
    \end{align}
    while for $P_{22}(t)$ we have the growth rate of its homogeneous part as
    \begin{align}\label{3.02}
        \lambda_{22} =  2\max \Re(\spec(-A_n))
    \end{align}
    By definition, $\Re(\spec(-A_n))\geq 0$ and $\Re(\spec(-A_a))<0$, hence $\lambda_{22}>\lambda_{12}$ and $\lambda_{22}\geq 0$. As the growth rate of $P_{22}$ is determined by $\lambda_{22}-\lambda_{12}$ and $B_{22}\succ0$, we have $P_{22}(t)\to\infty$ when $t\rightarrow\infty$. Next, we study the limit of $\tilde{P}$ to derive the limit of $\Sigma$. Let 
    $$
    \Sigma(t)=U\begin{pmatrix}
            \Sigma_{11}(t) & \Sigma_{12}(t)\\
            \Sigma_{21}(t) & \Sigma_{22}(t)
        \end{pmatrix}U^{\top}.$$
    By definition of the inverse matrix, we have $P_{21}(t) \Sigma_{12}(t) + P_{22}(t) \Sigma_{22}(t) = I$. As $\lim\limits_{t\rightarrow\infty}P_{22}(t)=\infty$, it follows that $\lim\limits_{t\rightarrow\infty}\Sigma_{22}(t)=0$; on the other hand, we have
    $$
    P_{11}(t) \Sigma_{12}(t) + P_{12}(t) \Sigma_{22}(t) = 0
    $$
    which implies that $\lim\limits_{t\rightarrow\infty}\Sigma_{12}(t)=0$ (because $P_{11}(t)\succ0$ for all $t\geq 0$ and $P^{\infty}_{11}\succ0$). 
    Finally, we have 
    $$
    \Sigma_{11} = (P_{11} - P_{12} P_{22}^{-1} P_{21})^{-1}
    $$
    then, by the estimates of $\lambda_{12}$ and $\lambda_{22}$ in \eqref{3.01} and \eqref{3.02}, the growth rate of $P_{12}P_{22}^{-1}P_{21}$ is controlled by $-\min(\Re(\spec(A_a)))<0$. Hence we have $\lim\limits_{t\rightarrow\infty}\Sigma_{11}(t)=(P^{\infty}_{11})^{-1}\succ0$. Defining $\Sigma_{\infty}^a=(P_{11}^{\infty})^{-1}$, one sees that $\Sigma_{\infty}^a$ solves the equation \eqref{eq:ARE}. By the symmetry of $\Sigma$, we also have $\lim\limits_{t\rightarrow\infty}\Sigma_{21}(t)=0$. The conclusion follows.
\end{proof}

\begin{remark}[Comparison with \cite{castin2025unified}]
    We compare our result to that in \cite[Proposition 4.8]{castin2025unified}. In their setting, $A = 0$, thus $A_a$ (and consequently $B_{11}$ and $\Sigma_\infty$) have no dimension. Then, we conclude that $\Sigma(t) \to 0$. From their result, one also concludes that $\mathrm{rank}(\Sigma_\infty) = 0$ and so $\Sigma_\infty = 0$. Overall, our result allows for non-zero $A$, while their result allows for indefinite $B$. 
\end{remark}

\subsection{Asymptotic dynamics for definite \texorpdfstring{$V$}{V}}\label{ss:generalV}

To extend our results for $V$ beyond scalar multiples of the identity, we study the case when $B$ and $V$ are negative or positive definite, and derive asymptotic results based on their sign.

Our first result concerns the spectrum of a matrix arising in the dynamics of $\mu$ in \eqref{eq:gaussianTransformer} at the equilibrium of $\Sigma$, which applies to general time-invariant $A$, $B$ and $V$. This result will be used later for the asymptotic analysis of the mean in \Cref{ss:meanMatching}.

\begin{lemma}\label{lem:zerospec}
Let $A,B,V \in \mathbb{R}^{d\times d}$ be arbitrary real matrices.  
Assume $\Sigma_\infty \succ 0$ satisfies the algebraic Bernoulli equation
\begin{equation}\label{eq:Bernoulli-Algebraic}
A\Sigma_\infty + \Sigma_\infty A^\top
+ V\Sigma_\infty B\Sigma_\infty
+ \Sigma_\infty B^\top \Sigma_\infty V^\top = 0.
\end{equation}
Then all eigenvalues of $M_\infty \coloneqq A + V \Sigma_\infty B$ are purely imaginary, i.e., $\Re\bigl(\spec(M_\infty)\bigr) \!=\!0$.
\end{lemma}
\begin{proof}
By definition, \eqref{eq:Bernoulli-Algebraic} implies
\begin{align}\label{eq:Lyap-transposed}
    M_\infty \Sigma_\infty + \Sigma_\infty M_\infty^\top = 0.
\end{align}
Let $\lambda \in \mathbb{C}$ be an eigenvalue of
$M_\infty$ with corresponding left eigenvector $v \in \mathbb{C}^{1\times d} \setminus \{0 \}$. Multiplying \eqref{eq:Lyap-transposed} by $v$ and $v^*$ yields
$v (M_\infty \Sigma_\infty  + \Sigma_\infty  M_\infty^\top) v^* = 0$.

Next, we use $v M_\infty  = \lambda v$ and $ M_\infty^\top v^* = \overline{\lambda} v^*$ to compute $(\lambda + \overline{\lambda}) \, v \Sigma_\infty v^* = 0$.

Because $\Sigma_\infty \succ 0$, we have $v\Sigma_\infty v^* > 0$, and hence $\lambda + \overline{\lambda} = 0$, which implies $\Re(\lambda) = 0$. Since $\lambda$ is an arbitrary eigenvalue of $M_\infty$, all eigenvalues of $M_\infty$ lie on the imaginary axis.
\end{proof}
\begin{remark}
    \Cref{lem:zerospec} implies that if the covariance converges to an equilibrium $\Sigma_{\infty} \succ 0$, then the linear matrix $M_\infty = A+V\Sigma_{\infty}B$ contributes only to an oscillatory part of the dynamics of $\mu$.
\end{remark}
Turning to the covariance dynamics, the following result characterizes its asymptotic behavior under opposite definiteness conditions on $V$ and $B$.
\begin{theorem}\label{thm:sigmaConvGeneralV}
    Let $\Sigma(t)$ be the solution of \eqref{eq:gaussianTransformer} with $\Sigma(0)=\Sigma_0\succ0$. Let $V\prec 0$ and $B \succ 0$. Then $\Sigma(t)$ is bounded for all $t\in[0,+\infty)$. Moreover:
    \begin{enumerate}
    \item If $\Re(\spec(A))>0$, then there exists an equilibrium $\Sigma_{\infty}\succ0$ solving the algebraic Bernoulli equation \eqref{eq:Bernoulli-Algebraic}.
    Further, if $\Sigma_\infty^{-1}V+V\Sigma_\infty^{-1}\prec0$, then $\Sigma_{\infty}$ is a locally stable equilibrium of $\Sigma(t)$.  
    
    \item If $\Re(\spec(A))\leq0$ and $\lambda_{\max}(A+A^\top) \le 0$, then $\lim\limits_{t \to \infty} \Sigma(t) = 0$ for all $\Sigma_0$.
    \item If $\Re(\spec(A))< 0$ and $\lambda_{\max}(A+A^\top) > 0$, then $\Sigma = 0$ is a locally stable equilibrium of the $\Sigma$ equation in \eqref{eq:gaussianTransformer}.
    \end{enumerate}
\end{theorem}
The proof is technical and thus relegated to \Cref{app:longProof}. Finally, when the signs of definiteness of $V$ and $B$ coincide, we have the following result showing the finite-time blow-up of the evolution of $\Sigma$.
\begin{theorem}
\label{prop:SigmaBlowUp}
    Let $\Sigma(t)$ be the solution of \eqref{eq:gaussianTransformer} with $\Sigma(0)=\Sigma_0\succ0$, where $V\succ 0$ and $B \succ 0$. Then 
    \begin{enumerate}
        \item If $\lambda_{\min}(A + A^\top) \ge 0$, then $\Sigma(t)$ blows up in finite time $T_f$, and for $\lambda_{\min}(A + A^\top) > 0$ we have
    \begin{align}\label{eq:blowuptime}
        T_f \le \frac{1}{\lambda_{\min}(A+A^\top)} \log\left( 1 + \frac{d\, \lambda_{\min}(A+A^\top)}{2 \lambda_{\min}(V) \lambda_{\min}(B) \tr(\Sigma_0)} \right).
    \end{align}
        \item If $\lambda_{\min}(A + A^\top)<0$, there exists $C>0$ such that when $\|\Sigma_0\|\ge C$, $\Sigma(t)$ blows up in finite time.
    \end{enumerate}
\end{theorem}
\begin{proof}
    First, consider the case $\lambda_{\min}(A + A^\top) \geq 0$. Differentiating $\tr(\Sigma)$ yields
\begin{align}
    \tr(\dot{\Sigma}) 
    =&~ 2\tr(A\Sigma) + 2\tr(\Sigma V \Sigma B)\nonumber\\
    \ge&~ \lambda_{\min}(A+A^{\top})\tr(\Sigma)+ 2 \lambda_{\min}(V)\lambda_{\min}(B) \tr(\Sigma^2)\nonumber\\
    \ge&~ \lambda_{\min}(A+A^{\top})\tr(\Sigma)+ \frac{2\lambda_{\min}(V)\lambda_{\min}(B)}{d}(\tr(\Sigma))^2, \label{eq:TraceEst}
\end{align}
where the inequalities are due to the fact that $V$, $B$ and $\Sigma$ are all positive definite, and that $\tr(\Sigma^2) \geq \frac{1}{d}(\tr(\Sigma))^2$. Therefore, as $\tr(\Sigma)>0$ and $\lambda_{\min}(A+A^{\top})\geq0$, it suffices to consider the following differential equation
\begin{align}\label{eq:TraceBound}
    \dot{y}=\lambda_{\min}(A+A^{\top})y + \frac{2\lambda_{\min}(V)\lambda_{\min}(B)}{d}y^2, \quad y(0)=\tr(\Sigma_0)
\end{align}
and when $\lambda_{\min}(A+A^{\top})>0$ solving it yields an explicit solution
$$
y(t) = \frac{\alpha y_0 e^{\alpha t}}{\alpha - \beta y_0 (e^{\alpha t} - 1)}
$$
where $\alpha=\lambda_{\min}(A+A^{\top})$, $\beta=\frac{2\lambda_{\min}(V)\lambda_{\min}(B)}{d}$. Therefore the blow-up time of $\tr(\Sigma)$ is shorter than that of $y$, namely,
$T_f\le \frac{1}{\alpha} \log\left( 1 + \frac{\alpha}{\beta y_0} \right)$, 
which is the estimate \eqref{eq:blowuptime}. On the other hand, if $\alpha=0$, then solving \eqref{eq:TraceBound} yields $y(t)=\frac{y_0}{1-y_0\beta t}$, confirming the finite-time blow-up of $\Sigma$.

Next, consider the case when $\lambda_{\min}(A + A^\top) < 0$. From \eqref{eq:TraceEst}\eqref{eq:TraceBound} one can see that when $\tr(\Sigma)>-\frac{ \lambda_{\min}(A+A^{\top})d}{2\lambda_{\min}(V)\lambda_{\min}(B)}$, $\tr(\Sigma)$ becomes monotonic with an increasing speed bounded from below by a positive constant, and hence blows up in finite time by the aforementioned analysis of the case $\lambda_{\min}(A + A^\top) > 0$.
\end{proof}
\begin{remark}
    Note that the boundedness of $\Sigma$ is determined by the sign of the quadratic term in \eqref{eq:gaussianTransformer}: when $V$ and $B$ have opposite signs of definiteness, by defining a Lyapunov function $\tr(\Sigma)$ we obtain the boundedness of $\Sigma$. Conversely, $\Sigma$ becomes unstable or blows up when $V$ and $B$ have the same sign of definiteness. From this argument, one can see that the proofs of \Cref{thm:sigmaConvGeneralV} and \Cref{prop:SigmaBlowUp} actually do not rely on the signs of $V$ and $B$, but on whether they coincide or not. 
\end{remark}

\subsection{Asymptotic mean matching}\label{ss:meanMatching}

We now study the ability of the Gaussian Transformer \eqref{eq:gaussianTransformer} to asymptotically match a prescribed mean. For $V = \eta I_d$, using the tools developed in \Cref{ss:VetaId} we show that, once the covariance converges, the mean can be steered to any desired target by suitable time-invariant parameters.

\begin{theorem}\label{thm:meanMatching}
Consider \eqref{eq:gaussianTransformer} with $A$ having the decomposition \eqref{eq:ADecomp}, suppose $V = \eta I_d$ and $B + B^\top\succ0$. 
\begin{enumerate}
    \item If $\eta< 0$, then $\lim\limits_{t\rightarrow\infty}\mu(t)=-(A+\eta\Sigma_{\infty}B+\eta I_d)^{-1}b$ where $\Sigma_{\infty}$ solves \eqref{eq:Bernoulli-Algebraic}.
    \item If $\eta> 0$ and $A_{a}$ exists, then $\|\mu(t)\|\to \infty$.
\end{enumerate}
\end{theorem}
\begin{proof}
    Define $M^\eta (t)\coloneqq A+\eta\Sigma(t) B$. Then, the resulting dynamics of $\mu$ is
    \begin{align}\label{eq:muDynamics}
        \dot{\mu}=(M^\eta (t)+\eta I_d)\mu+b,
    \end{align}
    where $\lim\limits_{t\rightarrow\infty}\Re(\spec(M^\eta(t)))\leq 0$. Indeed, by \Cref{thm:SigmaConvergenceVIdentity}, we have
    $$
    M_\infty^\eta \coloneqq \lim\limits_{t\rightarrow\infty}M^\eta(t)=\lim\limits_{t\rightarrow\infty}A+\eta\Sigma(t) B=U\begin{pmatrix}
        A_a+\Sigma_{\infty}^{a}B_{11} & A_{12}\\
        0 & A_n
    \end{pmatrix}U^{\top},
    $$
    where $U$ is an orthogonal matrix, $\Re(\spec(A_a))> 0$ and $\Re(\spec(A_n))\leq 0$. On the other hand, by \Cref{lem:zerospec}, $\Re(\spec(A_a+\Sigma^a_{\infty}B_{11}))=0$ due to the fact that $\Sigma_{\infty}^a$ solves \eqref{eq:ARE} which is a special case of \eqref{eq:Bernoulli-Algebraic}. Hence, we have $\Re(\spec(M_\infty^\eta)) \leq 0$.
    
    Define $\mu_{\infty}:=-(M_{\infty}^{\eta}+\eta I_d)^{-1}b$. Calculating the error dynamics yields
    \begin{align}\label{eq:muError}
        \frac{\dd}{\dd t}(\mu(t)-\mu_{\infty})= (M^{\eta}(t) +\eta I_d) (\mu(t)-\mu_{\infty}) + (M^{\eta}(t)-M^{\eta}_{\infty})\mu_{\infty}.
    \end{align}
    When $\eta<0$, choose the Lyapunov function as $\mathcal{V}(x):=x^{\top}Px$ where $P\succ 0$ solves the Lyapunov equation $(M_{\infty}^{\eta}+\eta I_d)^{\top}P+P(M_{\infty}^{\eta}+\eta I_d)=-I_d$, then one can use the fact that $M_\infty^\eta = \lim\limits_{t\rightarrow\infty}M^\eta(t)$ to show that the driftless dynamics $\dot{x}=(M^{\eta}(t)+\eta I_d)x$ exponentially decay to zero. Finally by applying Duhamel's principle to \eqref{eq:muError} and using the fact that $\lim\limits_{t\rightarrow\infty}M^{\eta}(t)-M_{\infty}^{\eta}=0$, we have $\lim\limits_{t\rightarrow\infty}\mu(t)=\mu_{\infty}$.

    On the other hand, if $\eta>0$ and there exists an eigenvalue of $A$ with positive real part, then $\eta I_d$ plus the purely oscillatory component $A_a+\Sigma_{\infty}^aB_{11}$ in $M_{\infty}^{\eta}$ contributes to an unstable force in the dynamics \eqref{eq:muDynamics}, driving $\mu(t)$ to infinity exponentially.
\end{proof}
\begin{remark}[On the role of non-zero bias]
In the special case $b=0$, the mean dynamics reduces to a homogeneous linear system. Thus, by \Cref{thm:meanMatching}, when $V = \eta I_d$, $\mu(t)\to0$ or $\mu(t)\to \infty$ as $t\to\infty$ for any  $A$ and any initial condition $\mu_0$. Hence a nonzero bias term $b$ is essential for asymptotic non-zero mean matching.
\end{remark}

\begin{remark}[On the definite $V$ case]
    Similarly, one can weaken the assumptions on $V$ by leveraging the results in \Cref{ss:generalV}. Indeed, combining the convergence of the covariance (\Cref{thm:sigmaConvGeneralV}) with \Cref{lem:zerospec}, we obtain that whenever $\Sigma(t)\to \Sigma_\infty$, the long-time behavior of $\mu(t)$ is determined by the spectrum of this limiting matrix, and more specifically, by the stability of $V$ plus a purely oscillatory component $A+V\Sigma_{\infty}B$.
\end{remark}

\section{Numerical experiments}\label{s:numerics}

In this section, we examine two aspects of the Gaussian Transformer dynamics \eqref{eq:gaussianTransformer}. First, we validate the asymptotic mean and covariance behavior predicted in \Cref{sec:asymptoticDynamics} in a controlled
two-dimensional setting. Second, we test whether an analogous Gaussian moment structure is visible in pretrained Transformer models when their inputs are sampled from a Gaussian distribution. Codes to reproduce our results can be found in \url{https://github.com/DCN-FAU-AvH/gaussianTransformers}.

\subsection{Validation of asymptotic results}

We begin with a two-dimensional example illustrating the asymptotic behavior of \Cref{eq:gaussianTransformer} described in \Cref{sec:asymptoticDynamics}. We fix $B \succ 0$ and $V = - I_d$, and consider three choices of $A$: one with $\Re(\spec(A))>0$, one with $\Re(\spec(A))\leq 0$, and one whose eigenvalues have real parts of mixed sign.
We set $\mu_0=(-1,-1)^\top$, $\Sigma_0=
    \left(\begin{smallmatrix}
        0.4 & 0.2\\
        0.2 & 0.2
    \end{smallmatrix}\right)$,
and prescribed limiting mean as $\mu_\infty=(1,2)^\top$. In each case, the bias term $b$ is chosen according to the mean-matching condition
of \Cref{thm:meanMatching}, so that $\mu_\infty$ is an equilibrium for the mean
dynamics.
\begin{figure}[h]
    \centering
    \includegraphics[
        width=0.85\textwidth,
        trim={0cm 0.5cm 0cm 0cm},
        clip
    ]{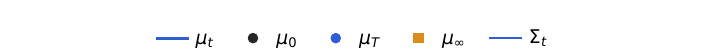}

    \begin{subfigure}{0.3\linewidth}
        \centering
        \includegraphics[width=\linewidth]{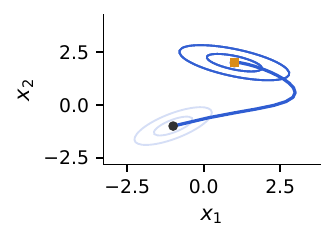}
        \caption{$\Re(\spec(A))>0$}
    \end{subfigure}\hfill
    \begin{subfigure}{0.3\linewidth}
        \centering
        \includegraphics[width=\linewidth]{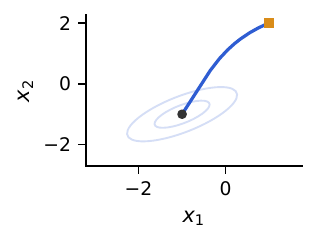}
        \caption{$\Re(\spec(A))\leq 0$}
    \end{subfigure}\hfill
    \begin{subfigure}{0.3\linewidth}
        \centering
        \includegraphics[width=\linewidth]{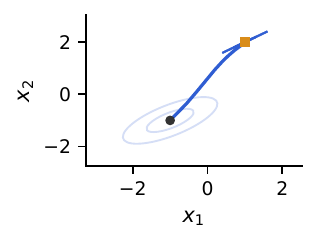}
        \caption{mixed sign}
    \end{subfigure}
    \caption{
    Asymptotic dynamics of \Cref{eq:gaussianTransformer} with $B\succ 0$ and
    $V = -I_d$. The bias is chosen so that $\mu_\infty=(1,2)^\top$ is the limiting mean.}
    \label{fig:meanMatching}
\end{figure}

The resulting trajectories are shown in \Cref{fig:meanMatching}. In all three regimes, the mean converges to the prescribed target by time $T=100$. The covariance behavior depends on the spectral regime of $A$, as predicted by \Cref{thm:SigmaConvergenceVIdentity}. For $\Re(\spec(A))>0$, the covariance remains nondegenerate and converges toward a finite limiting shape. When $\Re(\spec(A))\leq 0$, the covariance collapses. For the mixed-sign case, the evolution is anisotropic: contraction occurs only along the stable directions. 

\subsection{Gaussian structure in pretrained Transformers}

We next examine whether Gaussian moment structure persists in pretrained Transformers. For a model with embedding dimension $d$, we draw independent input tokens $x_i^{(0)}\sim \mathcal N(0,I_d)$, propagate them through the layers, and denote by $\rho^{(\ell)}$ their empirical
distribution at layer $\ell$. Further, let $\gamma^{(\ell)} = \mathcal N\left(\mu^{(\ell)},\Sigma^{(\ell)}\right)$ be the Gaussian distribution with the same empirical mean and covariance as $\rho^{(\ell)}$.
\begin{figure}[t]
\centering
\begin{subfigure}{.49\textwidth}
    \centering
    \includegraphics[width=\textwidth]{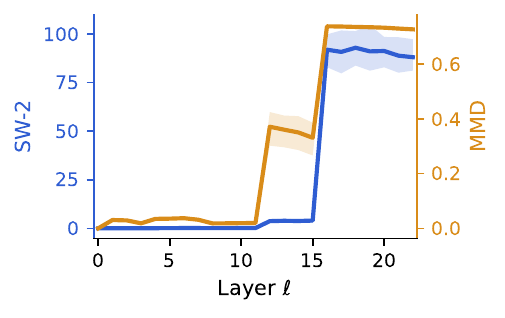}
    \caption{ModernBERT}
    \label{fig:momentMatched-modernBERT}
\end{subfigure}
\hfill
\begin{subfigure}{.49\textwidth}
    \centering
    \includegraphics[width=\textwidth]{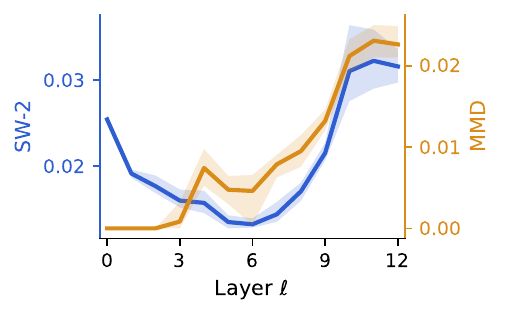}
    \caption{Tiny-DeiT}
    \label{fig:momentMatched-tinyDeiT}
\end{subfigure}
\caption{Distances between the empirical distribution $\rho^{(\ell)}$ and its moment-matched Gaussian approximation $\gamma^{(\ell)}$. Curves show the mean across $20$ batches of $n=8192$ tokens, and shaded regions indicate the empirical $0.1$--$0.9$ quantile band.}
\label{fig:momentMatched}
\end{figure}
\begin{figure}[h]
    \centering
        \includegraphics[width=1\linewidth]{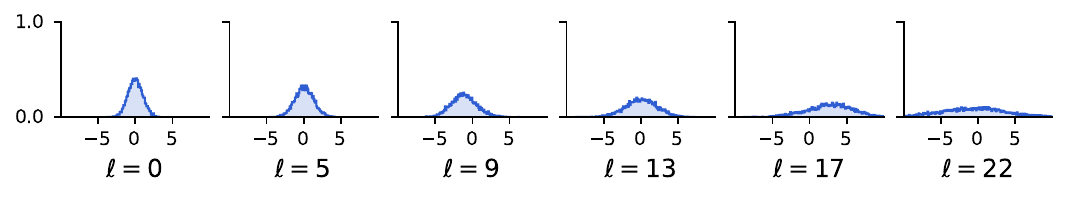}
        \includegraphics[width=1\linewidth]{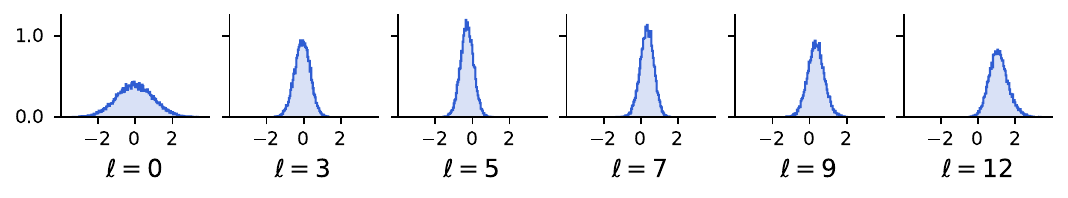}    
    \caption{Marginal distributions across layers of ModernBERT (top row) Tiny-DeiT (bottom row). Each panel shows the coordinate whose variance is closest to the median at that layer, using $n=8\,192$ Gaussian
    input tokens.}
    \label{fig:repCoord}
\end{figure}
We compare $\rho^{(\ell)}$ with $\gamma^{(\ell)}$ using the sliced Wasserstein distance (SW-2):
$$
\mathrm{SW}_2(\rho^{(\ell)},\gamma^{(\ell)})
=
\left(
\mathbb{E}_{\theta}
W_2^2\!\left(
\langle \theta,x^{(\ell)}\rangle,
\langle \theta,z^{(\ell)}\rangle
\right)
\right)^{\frac{1}{2}},
$$
where $x^{(\ell)}\sim \rho^{(\ell)}$ and
$z^{(\ell)}\sim\gamma^{(\ell)}$, and the maximum mean discrepancy (MMD):
\[
\mathrm{MMD}^2(\rho^{(\ell)},\gamma^{(\ell)})
=
\mathbb{E}_{x,x'\sim\rho^{(\ell)}}[k(x,x')]
+
\mathbb{E}_{z,z'\sim\gamma^{(\ell)}}[k(z,z')]
-
2\mathbb{E}_{x\sim\rho^{(\ell)},\,z\sim\gamma^{(\ell)}}[k(x,z)],
\]
where $k$ is the Gaussian kernel. For our experiments, we use ModernBERT~\cite{warner2025smarter}, a pretrained language encoder with $L=22$ layers and $d=768$, and Tiny-DeiT \cite{pmlr-v139-touvron21a}, a compact vision Transformer with $L=12$ layers and $d=192$. The resulting distances across layers are shown in \Cref{fig:momentMatched}. For both architectures, the distance to the moment-matched Gaussian remains comparatively small through the early and intermediate layers. The discrepancies increase in deeper layers, indicating that non-Gaussian behavior abruptly emerges at later stages of the network.

We further assess Gaussian preservation in pretrained models by visualizing one-dimensional marginals. At each selected layer, we choose the coordinate whose marginal variance is closest to the median. As shown in \Cref{fig:repCoord}, the two architectures behave qualitatively differently. In ModernBERT, the marginal broadens with depth and becomes progressively flatter. In contrast, Tiny-DeiT maintains a more Gaussian-like profile across layers, with only modest changes in spread. 

Finally, we test whether the covariance regimes predicted by \Cref{thm:sigmaConvGeneralV} appear in a realistic setting. Starting from Tiny-DeiT, we remove normalization layers and replace self-attention with residual single-head attention maps using prescribed matrices $V$ and $B$, while leaving the nonlinear feed-forward blocks unchanged. For Gaussian inputs, we track the empirical mean and covariance trace.

\Cref{fig:modifiedTinyDeiT} shows two distinct regimes: opposing signs of $V$ and $B$ keep the covariance trace bounded, whereas jointly positive signs cause rapid growth. Thus, the sign structure in \Cref{thm:sigmaConvGeneralV} predicts qualitative covariance behavior even in a discrete residual architecture with nonlinear feed-forward blocks.
\begin{figure}
\centering
\begin{subfigure}{.49\textwidth}
    \centering
    \includegraphics[width=\textwidth]{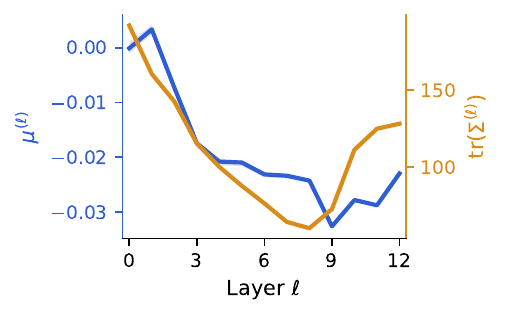}
    \caption{$V\prec 0$, $B \succ 0$}
    \label{fig:modifiedTinyDeiT-covariance-opposite}
\end{subfigure}
\hfill
\begin{subfigure}{.49\textwidth}
    \centering
    \includegraphics[width=\textwidth]{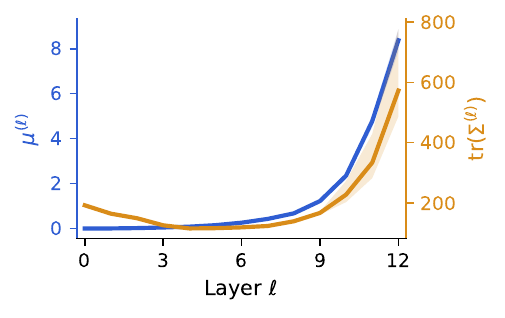}
    \caption{$V\succ 0$, $B \succ 0$}
    \label{fig:modifiedTinyDeiT-covariance-same}
\end{subfigure}
\caption{Evolution of the empirical mean and covariance trace in modified Tiny-DeiT blocks. Each curve reports the average over 20 independent batches, each containing $n=8192$ tokens initialized from a standard Gaussian distribution. Shaded regions indicate the empirical $0.1$--$0.9$ quantile band.}
\label{fig:modifiedTinyDeiT}
\end{figure}

\section{Conclusions}\label{sec:conclusions}

In this work, we connected modern Transformer architectures with classical control theory by analyzing a mean-field formulation in the Gaussian regime. We proved that, for self-attention with affine feed-forward layers, Gaussian measures are preserved by the induced flow. This invariance reduces the nonlocal transport PDE on probability measures to a finite-dimensional bilinear control system for the mean and covariance.

Our characterization of mean and covariance dynamics provides a control theoretic interpretation of neural network expressivity. Finite-time reachability yields a minimal interpolation property within the invariant class of Gaussian measures with fixed covariance rank, while the asymptotic analysis identifies parameter regimes leading either to stable covariance dynamics or to finite-time blow-up.

At the model level, covariance instabilities correspond to divergence of token values during forward propagation and therefore suggest possible numerical failure modes. The experiments support this picture: although exact Gaussian invariance is not expected in trained encoder Transformers, Gaussian moment structure remains sufficiently persistent in early and intermediate layers for the reduced dynamics to capture relevant qualitative behavior.

Several questions remain open. First, the long-time analysis relies on structural assumptions on the Transformer parameter matrices, including sign and symmetry conditions on $B$ and $V$. Extending the theory beyond these regimes is challenging because the resulting Bernoulli-type matrix equations may exhibit non-normal behavior. Second, the present framework treats a simplified architecture: it focuses on single Gaussian input distributions, omits layer normalization, and uses affine feed-forward layers to preserve Gaussian closure. The restriction to a single Gaussian is essential: for Gaussian mixtures, the exponential factors produced by the different components no longer cancel, so the attention field is no longer affine in $x$, and the closed Bernoulli system for the mean and covariance no longer applies. Consequently, finite Gaussian mixtures do not form an invariant class under the Transformer flow, and no exact finite-dimensional closure analogous to \Cref{prop:gaussianTransformer} should be expected. 

A further perspective is to study the infinite inverse-temperature limit within the Gaussian setting, obtained by replacing $B$ in the attention weights with $\beta B$ and letting $\beta \to \infty$. For every finite $\beta$, the mean and covariance still satisfy a Bernoulli-type system, with the quadratic terms scaled by $\beta$. The limit formally corresponds to hardmax attention, but this is not directly well posed on the unbounded support of Gaussians. Whether a meaningful interpretation survives after truncation, renormalization, or projection onto Gaussian moments remains an interesting open problem. More generally, extensions to Gaussian mixtures, multi-head attention, normalization mechanisms, nonlinear feed-forward layers, and singular attention limits would require a more delicate analysis, likely based on projected or approximate moment-closure methods rather than exact Gaussian closure.

Overall, the Gaussian Transformer shows that expressivity and forward-pass stability can be studied jointly through controlled mean--covariance dynamics. This perspective offers a principled route toward the design and analysis of stable, controllable, and mathematically grounded Transformer architectures.

\appendix
\crefalias{section}{appendix}

\section{Proofs of technical results}
This appendix contains proofs of auxiliary results. Throughout, we omit the domains of integration to ease readability.

\subsection{Proof of \texorpdfstring{\Cref{lem:subGaussian}}{Lemma 2.2}}\label{app:lemProof}

\begin{proof}
    Let $T^*:=\sup\{T\ge 0:\ E_t\le 2E_0 \text{ for all } t\in[0,T]\}$. We shall show $T^* > 0$. 
    
    Define $\Lambda_t(\xi) = \log \int e^{\xi^\top y} \dd\rho_t(y)$ as the cumulant generating function of $\rho_t$. Then, it holds that $\mathcal{A}_{\rho_t}(x) = V \nabla_\xi \Lambda_t(Bx)$. As $E_t \leq 2 E_0$, by Young's inequality we have
    \begin{equation}\label{eq:linearGrowthBound}
    \Lambda_t(\xi) \le \log \left( e^{\frac{1}{4\kappa_0}|\xi|^2} \int e^{\kappa_0|y|^2} \dd\rho_t(y)\right) \le \log\left(e^{\frac{1}{4\kappa_0}|\xi|^2} \cdot 2E_0 \right)\le  \frac{1}{4\kappa_0}|\xi|^2 + \log(2E_0)
    \end{equation}
    for all $t\in[0,T^*]$. Note that \eqref{eq:linearGrowthBound} implies that the moment generating function of $\rho_t$ is finite on all of $\mathbb{R}^d$, hence $\Lambda_t$ is well defined and smooth in $\xi$, and
    $$
    \nabla \Lambda_t(\xi)
    =
    \frac{\int y\,e^{\xi^\top y}\,\dd \rho_t(y)}{\int e^{\xi^\top y}\,\dd\rho_t(y)}.
    $$
    Since $\Lambda_t$ is convex and has at most quadratic growth, there exist constants $C_1,C_2>0$ such that
    $\|\nabla \Lambda_t(\xi)\|\le C_1|\xi|+C_2$
    for all $t\in[0,T^*]$. It follows that
    $$
    \|\mathcal A[{\rho_t}](t,x)\|
    =
    \|V\nabla \Lambda_t(Bx)\|
    \le \|V\|\bigl(C_1\|B\||x|+C_2\bigr).
    $$
    Since $\sigma=\mathrm{ReLU}$ is globally Lipschitz and has linear growth, the velocity field
    $u_t(x):=\sigma(Ax+b)+\mathcal A[\rho_t](t,x)$
    is locally Lipschitz in $x$ and satisfies
    \begin{equation}\label{eq:linearGrowth_u}
            \|u_t(x)\|\le K_1|x|+K_2
    \end{equation}
    for all $t\in[0,T^*]$ and some constants $K_1,K_2>0$. Therefore, by the Picard--Lindel\"{o}f existence-uniqueness theorem for ODEs \cite{Teschl2012ODE}, for each $x_0\in\mathbb R^d$ the characteristic equation
    $$
    \dot\Phi_t(x_0)=u_t(\Phi_t(x_0)),\qquad \Phi_0(x_0)=x_0,
    $$
    admits a unique solution on $[0,T^*]$, and the linear-growth bound \eqref{eq:linearGrowth_u} prevents finite-time blow-up. Since $\rho_t$ solves the continuity equation \eqref{eq:transformerPDE} with velocity field $u_t$, the method of characteristics yields $\rho_t=(\Phi_t)_\#\rho_0$ where $(\Phi_t)_\#$ is the pushforward of the flow map $\Phi_t$. Consequently,
    \begin{align}\label{eq:subGaussianInt}
    E_t=\int e^{\kappa_0 |\Phi_t(x_0)|^2}\,\dd\rho_0(x_0).
    \end{align}
    Thus, for a.e. $t\in[0,T^*]$, 
    \begin{align}\label{eq:EstPhi}
        \frac{\dd}{\dd t} |\Phi_t(x_0)|\le\|\sigma(A\Phi_t(x_0)+b)+\mathcal{A}[\rho_t](t,\Phi_t(x_0))\| \le K_1 |\Phi_t(x_0)| + K_2.
    \end{align}
    Multiplying by  $e^{-K_1 t}$ we have 
    $\frac{\dd}{\dd t} \left( |\Phi_t(x_0)| e^{-K_1 t} \right) \le K_2 e^{-K_1 t}$,
    and integrating both sides from $0$ to $t$ yields
    $|\Phi_t(x_0)| \le |x_0| e^{K_1 t} + \frac{K_2}{K_1} \left( e^{K_1 t} - 1 \right)$,
    therefore, for any $\epsilon>0$,
    $$
    |\Phi_t(x_0)|^2 \le (1+\epsilon) e^{2K_1 t} |x_0|^2 + \left(1+\frac{1}{\epsilon}\right) \left( \frac{K_2}{K_1} (e^{K_1 t} - 1) \right)^2.
    $$
    Then substituting it into \eqref{eq:subGaussianInt} we have
    \begin{align}\label{eq:subGaussianEst}
        E_t \le \exp \left(\kappa_0\left( 1+\tfrac{1}{\epsilon} \right) \left( \tfrac{K_2}{K_1} (e^{K_1 t} - 1) \right)^2\right)  \int \exp\left( \kappa_0 (1+\epsilon) e^{2K_1 t} |x_0|^2 \right) \dd\rho_0(x_0).
    \end{align}
    Since $\rho_0=\mathcal N(\mu_0,\Sigma_0)$ and $\kappa_0<\frac{1}{2\lambda_{\max}(\Sigma_0)}$, $\int e^{\alpha |x_0|^2}\,\dd\rho_0(x_0)$ is finite for every $\alpha<\frac{1}{2\lambda_{\max}(\Sigma_0)}$, and depends continuously on $\alpha$. Hence one can choose $\epsilon>0$ and then $\tau>0$ sufficiently small so that
    $\kappa_0(1+\epsilon)e^{2K_1 t}<\frac{1}{2\lambda_{\max}(\Sigma_0)}$
     for all $t\in[0,\tau]$,
    and the right-hand side of \eqref{eq:subGaussianEst} is bounded by $\frac{3}{2}E_0$ for all $t\in[0,\tau]$. Therefore $E_t\le 2E_0$ on $[0,\tau]$, so $T^*\ge \tau>0$.
\end{proof}

\subsection{Proof of \texorpdfstring{\Cref{prop:GaussianPreserv}}{Proposition 2.3}}\label{app:proProof}

\begin{proof}
    Consider $X_t\sim{\rho}_t$, $Y_t\sim\nu_t$ on $\R^d$. Then we have
    \begin{align*}
    \dot{X}_t=\sigma(AX_t+b)+\mathcal{A}[\rho_t](t, X_t), \quad\quad
        \dot{Y}_t=AY_t+b+\mathcal{A}[\nu_t](t, Y_t)
    \end{align*}
    Let $e_t = X_t - Y_t$ be the discrepancy and use the shorthand notation $\mathcal{A}[\rho_t](t, X_t) = \mathcal{A}_{\rho_t}(X_t)$ for readability. Differentiation yields
    \begin{align}\label{eq:DiscrDiff}
        \frac{\dd}{\dd t}{e}_t = \eta(AX_t+b) + A e_t + \left( \mathcal{A}_{\rho_t}(X_t) - \mathcal{A}_{\nu_t}(Y_t) \right),
    \end{align}
    where $\eta(x) =\sigma(x)-x= \max\{0, -x\}$. We shall bound the term
    \begin{align*}
        \mathcal{A}_{\rho_t}(X_t) - \mathcal{A}_{\nu_t}(Y_t)=\mathcal{A}_{\rho_t}(X_t) - \mathcal{A}_{{\nu}_t}(X_t)+\mathcal{A}_{{\nu}_t}(X_t) - \mathcal{A}_{\nu_t}(Y_t)
    \end{align*}    
    by calculating the spatial and measure variation, respectively. For the spatial variation we use that $\nu_t$ is Gaussian and the expression \eqref{eq:SAonlyVectorField}, hence
    $\mathcal{A}_{\nu_t}(X_t) - \mathcal{A}_{\nu_t}(Y_t) = V\Sigma_t B (X_t - Y_t)$.
    Taking the $L^2$ norm directly yields the spatial bound:
    $$
    \|\mathcal{A}_{\nu_t}(X_t) - \mathcal{A}_{\nu_t}(Y_t)\|_{L^2} \le \|V\Sigma_t B\| \|X_t - Y_t\|_{L^2}.
    $$
    As for the measure variation, we define for any density $\varrho$ on $\R^d$ two functions $N_\varrho(x): = \int y e^{y^\top B x} \dd\varrho(y)$, $Z_\varrho(x): = \int e^{y^\top B x} \dd\varrho(y)$. Then $\mathcal{A}_{\varrho}(x)=V\frac{N_{\varrho}(x)}{Z_{\varrho}(x)}$. We have
    \begin{align}\label{eq:SAdifference}
        \mathcal{A}_{\rho_t}(x) - \mathcal{A}_{\nu_t}(x) = \frac{V}{Z_{\rho_t}(x)} \left(N_{\rho_t}(x) - N_{\nu_t}(x)\right) + \mathcal{A}_{{\rho}_t}(x) \frac{Z_{\nu_t}(x) - Z_{\rho_t}(x)}{Z_{\rho_t}(x)}.
    \end{align}
    Let $\gamma_t$ be the optimal coupling of $\rho_t$ and $\nu_t$ \cite{villani2009optimal}, which is a probability distribution on $\R^{2d}$ satisfying $\iint \|y-z\|^2 \dd\gamma_t = W_2^2(\rho_t, \nu_t)$. Then we have
    \begin{align*}
        |Z_{\rho_t}&(x) - Z_{\nu_t}(x)| 
        \le \iint \Big( \|B\| \|x\| \|y-z\| \Big) \Big( e^{y^\top B x} + e^{z^\top B x} \Big) \dd\gamma_t(y,z)\\
        &\le\|B\| \|x\| \left( \iint \|y-z\|^2 \dd\gamma_t(y,z) \right)^{\frac{1}{2}} \left( \iint \left( e^{y^\top B x} + e^{z^\top B x} \right)^2 \dd\gamma_t(y,z) \right)^{\frac{1}{2}}\\
        &\le \|B\| \|x\| W_2(\rho_t, \nu_t) \left( 2\int e^{2y^\top B x} \dd\rho_t(y) + 2\int e^{2z^\top B x} \dd\nu_t(z) \right)^{\frac{1}{2}}.
    \end{align*}
    On the other hand, we have 
    \begin{align*}
    &\|N_{\rho_t}(x) - N_{\nu_t}(x)\|
    \le \iint \|y-z\| e^{y^\top B x} \dd\gamma_t
    + \iint \|z\| \left| e^{y^\top B x} - e^{z^\top B x} \right| \dd\gamma_t\\
    &\le
    \left( \iint \|y-z\|^2 \dd\gamma_t \right)^{\frac{1}{2}}
    \left( \iint e^{2y^\top B x} \dd\gamma_t \right)^{\frac{1}{2}} \\
    &\quad
    + \iint \|z\| \Big( \|B\| \|x\| \|y-z\| \Big)
    \Big( e^{y^\top B x} + e^{z^\top B x} \Big) \dd\gamma_t\\
    &\le
    W_2(\rho_t, \nu_t)
    \left( \int e^{2y^\top B x} \dd\rho_t(y) \right)^{\frac{1}{2}} \\
    &\quad
    + \|B\| \|x\| W_2(\rho_t, \nu_t)
    \left(
    \iint 2\|z\|^2 e^{2y^\top B x} \dd\gamma_t
    + \iint 2\|z\|^2 e^{2z^\top B x} \dd\gamma_t
    \right)^{\frac{1}{2}}.
    \end{align*}
    Additionally, by Jensen's inequality, $
    Z_{\rho_t}(x) \ge \exp\left( \int y^\top B x \, \dd\rho_t(y) \right) \ge \exp\left( \mu_\rho^\top B x \right), 
    $ 
    where $\mu_\rho = \int y \, \dd\rho_t(y)$. Substituting the above estimates of $Z_{\rho_t}(x)-Z_{\nu_t}(x)$, $N_{\rho_t}(x)-N_{\nu_t}(x)$ and $\frac{1}{Z_{\rho_t}(x)} \le \exp\left( -\mu_\rho^\top B x \right)$ into \eqref{eq:SAdifference}, we obtain
    \begin{align*}
        \|\mathcal{A}_{\rho_t}(x) - \mathcal{A}_{\nu_t}(x)\|\le& \frac{\|V\|}{Z_{\rho_t}(x)} \left( \|N_{\rho_t}(x) - N_{\nu_t}(x)\| + \frac{\|N_{\nu_t}(x)\|}{Z_{\nu_t}(x)} |Z_{\rho_t}(x) - Z_{\nu_t}(x)| \right)\\
        \le& {\|V\|}\Big[ I_t^1 + I_t^2 + I_t^3 \Big] W_2(\rho_t, \nu_t)
    \end{align*}
    where 
    \begin{align*}
        I_t^1&=\left( \int e^{2(y-\mu_{\rho})^\top B x} \dd\rho_t(y) \right)^{\frac{1}{2}},\\
        I_t^2&=\|B\| \|x\| \left( \iint 2\|z\|^2 e^{2(y-\mu_{\rho})^\top B x} \dd\gamma_t + \iint 2\|z\|^2 e^{2(z-\mu_{\rho})^\top B x} \dd\gamma_t \right)^{\frac{1}{2}},\\
        I_t^3&= \|B\| \|x\|  \left( 2\int e^{2(y-\mu_{\rho})^\top B x} \dd\rho_t + 2\int e^{2(z-\mu_{\rho})^\top B x} \dd\nu_t \right)^{\frac{1}{2}}(\|\mu_t\| + \|\Sigma_t B\| \|x\|).
    \end{align*}
     We shall show that the $L^2(\rho_t)$ norm of $I_t^1+I_t^2+I_t^3$ is finite within $[0,T^*]$. 
     Since $z\sim \nu_t$, terms involving $\int \|z\|^2 e^{c\|z\|} \dd\nu_t(z)$, for $c > 0$, are bounded over $t\in[0,T^*]$, and polynomials and exponentials of $x$ integrated against $\rho_t(x)$ are also finite over $t\in[0,T^*]$ due to the sub-Gaussianity proved in \Cref{lem:subGaussian}. Thus, by the algebraic inequality $(a+b+c)^2 \le 3(a^2 + b^2 + c^2)$, it suffices to consider the worst-case integrals
    \begin{align*}
        J^1_t=&\int  \|x\|^2 \iint \|z\|^2 e^{2(y-\mu_\rho)^\top B x} \dd\gamma(y,z) \dd\rho_t(x) \\
        =& \iiint \|x\|^2 \|z\|^2 e^{2(y-\mu_\rho)^\top B x} \dd\gamma(y,z) \dd\rho_t(x),\\
    J^2_t=&\int \|x\|^4 \int e^{2(y-\mu_\rho)^\top B x} \dd\rho_t(y) \dd\rho_t(x) = \iint \|x\|^4 e^{2(y-\mu_\rho)^\top B x} \dd\rho_t(y) \dd\rho_t(x)
    \end{align*}
    and show that they are both bounded within a short time. By Fubini's theorem and the Cauchy--Schwarz inequality, 
    \begin{align*}
        J_t^1 \le & {\left( \int \|x\|^2 e^{\|B\| \|x\|^2} \dd\rho_t(x) \right)}  {\left( \iint \|z\|^2 e^{\|B\| \|y-\mu_\rho\|^2} \dd\gamma_t(y,z) \right)},\\
        J_t^2 \le &  {\left( \int \|x\|^4 e^{\|B\| \|x\|^2} \dd\rho_t(x) \right)} {\left( \int e^{\|B\| \|y-\mu_\rho\|^2} \dd\rho_t(y) \right)}.
    \end{align*}
    Since $\|B\|<\kappa_0$, by \Cref{lem:subGaussian}, there exists a constant $c_0$ such that
    \begin{multline*}
            \max\Big\{\left( \int\|x\|^2 e^{\|B\| \|x\|^2} \dd\rho_t(x) \right), \\
            \left(\int \|x\|^4 e^{\|B\| \|x\|^2} \dd\rho_t(x)\right),
            \left( \int e^{\|B\| \|y-\mu_\rho\|^2} \dd\rho_t(y) \right)\Big\}\le c_0
    \end{multline*}
    for all $t\in[0,T^*]$. As for the second part of $J^1_t$, we have
    \begin{align*}
        \iint \|z\|^2 e^{\|B\| \|y-\mu_\rho\|^2} \dd\gamma_t(y,z) \le  \left( \int \|z\|^4 \dd\nu_t(z) \right)^{\frac{1}{2}} \left( \int e^{2\|B\| \|y-\mu_\rho\|^2} \dd\rho_t(y) \right)^{\frac{1}{2}}\\
        \le\left( \int \|z\|^4 \dd\nu_t(z) \right)^{\frac{1}{2}} \left( e^{4\|B\| \|\mu_\rho\|^2} \int e^{4\|B\| \|y\|^2} \dd\rho_t(y)\right)^{\frac{1}{2}},
    \end{align*}
    which is finite over $t\in[0,T^*]$ thanks to \Cref{lem:subGaussian}.
    Summarizing the above argument, we have
    $\|\mathcal{A}_{\rho_t}(x) - \mathcal{A}_{\nu_t}(x)\|\le L_{\rho}(t) W_2(\rho_t,\nu_t)$
    where $L_{\rho}(t)$ is bounded over $t\in[0,T^*]$. Hence by \eqref{eq:DiscrDiff}, we have
    $$
    \frac{\dd}{\dd t} \|e_t\|_{L^2} \le \Big( \|A\| + \|V\Sigma_tB\| + L_\rho(t)  \Big) \|e_t\|_{L^2} + \|\eta(AX_t+b)\|_{L^2(\rho_t)}.
    $$
    Since $L_{\rho}(t)$ and $\|V\Sigma_tB\|$ are uniformly bounded over $t\in[0,T^*]$, by Gr\"{o}wall's lemma, there exists $K>0$ such that
    \begin{align}\label{eq:Distance}
        W_2(\rho_t, \nu_t) \le \|e_t\|_{L^2} \le \int_0^t \|\eta(AX_s+b)\|_{L^2(\rho_s)} e^{K(t-s)} \dd s.
    \end{align}
    By the definition of the pushforward measure, we have
    \begin{align*}
        \|\eta(AX_s+b)\|_{L^2(\rho_s)} =& \|\eta(A \Phi_s(X_0) + b)\|_{L^2(\rho_0)}\\
        \le& \|\eta(AX_0 + b)\|_{L^2(\rho_0)} + \|A\| \|\Phi_s(X_0) - X_0\|_{L^2(\rho_0)}.
    \end{align*}
    As we have established in \eqref{eq:EstPhi} that 
    $\frac{\dd}{\dd t} |\Phi_t(x_0)| \le K_1 |\Phi_t(x_0)| + K_2$,
    it follows that 
    $$
    |\Phi_s(X_0) - X_0| \le (K_1 |X_0| + K_2) \int_0^s e^{K_1 r} \dd r \le (K_1 |X_0| + K_2) s e^{K_1 s}.
    $$
    Therefore we have $\|\eta(AX_s+b)\|_{L^2(\rho_s)} \le \|\eta(AX_0 + b)\|_{L^2(\rho_0)} + C_0 s e^{K_1 s}$, where $C_0=\|A\|(K_1\sqrt{\|\mu_0\|^2+\tr(\Sigma_0)}+K_2)$. Substituting into \eqref{eq:Distance}, yields
    \begin{align*}
        W_2(\rho_t, \nu_t) \le &\int_0^t \Big[ \|\eta(AX_0 + b)\|_{L^2(\rho_0)} + C_0 s e^{K_1 s} \Big] e^{K(t-s)} \dd s\\
        \le & \|\eta(AX_0 + b)\|_{L^2(\rho_0)} \left( \frac{e^{Kt} - 1}{K} \right)+ C_0 e^{Kt} \int_0^t s e^{(K_1 - K)s} \dd s\\
        \le &~t \cdot \left\| \eta(AX_0 + b)\right\|_{L^2(\rho_0)} + \mathcal{O}(t^2)
    \end{align*}
    for all $t\in[0,T^*]$, which completes the proof.
\end{proof}

\subsection{Proof of \texorpdfstring{\Cref{thm:sigmaConvGeneralV}}{Theorem 4.5}}\label{app:longProof}

\begin{proof}
We split the proof in two cases, depending on the sign of $\Re(\spec(A))$. 

\paragraph{Case $\Re(\spec(A)) > 0$} We first show the existence of an equilibrium of the $\Sigma$ equation in \eqref{eq:gaussianTransformer} (i.e., the solution of \eqref{eq:Bernoulli-Algebraic}) using the Poincar\'{e}-Hopf theorem (see \cite[Chapter 7]{milnor1997topology}). More precisely, we construct a domain $\mathcal{K}\subset\R^{d \times d}$ making sure that the vector field defining the $\Sigma$ equation in \eqref{eq:gaussianTransformer} is transversal to $\partial\mathcal{K}$, and then show the existence of equilibria inside $\mathcal{K}$ by a topological argument.

We begin by constructing the domain $\mathcal{K}$ using hyperplanes defined via Lyapunov functions. Calculating the trace of the derivative of $\Sigma(t)$ yields
$
\tr(\dot{\Sigma}) = \tr(A\Sigma + \Sigma A^\top) + \tr(V\Sigma B \Sigma + \Sigma B^\top \Sigma V^\top)= 2\tr(A\Sigma) + 2\tr(\Sigma V \Sigma B).
$ 
We shall compare the growth of the two terms in this equation to give a sign to $\tr(\dot{\Sigma})$ when the norm of $\Sigma$ is large. On the one hand, there exists $C_1\geq 0$ such that $\tr(A\Sigma) \leq  C_1 \|\Sigma\|_F$ where $\|\cdot\|_{F}$ is the Frobenius norm of square matrices. Meanwhile, since $\Sigma V\Sigma\prec0$ (by Sylvester's law of inertia) and $B\succ0$, we have $\tr(\Sigma V \Sigma B)<0$, and further by applying von Neumann's trace inequality (refer to, for example, \cite[Theorem 4.3.53]{horn2012matrix}), we have
\begin{align*}
    \tr(\Sigma V \Sigma B)\leq  \lambda_{\min}(B)\tr(\Sigma V \Sigma)
    \leq \lambda_{\min}(B)\lambda_{\max}(V)\tr(\Sigma^2)
    \leq -C_2\|\Sigma\|^2_F
\end{align*}
where $C_2 > 0$. Therefore, since $\Sigma\succ0$ by \Cref{lem:RankPreserv}, there exists a constant $R>0$ such that if $\tr(\Sigma)=R$, then $\tr(\dot{\Sigma})<0$. Next, denote $x\coloneqq \mathrm{vec}(\Sigma)$. Consider a boundary function $U(\Sigma) = x^\top P x = \text{vec}(\Sigma)^\top P \text{vec}(\Sigma)$ where $P\succ0$ solves the Lyapunov equation 
\begin{align*}
    P(I_d \otimes A + A \otimes I_d)+(I_d \otimes A + A \otimes I_d)^{\top}P=I_{d^2}.
\end{align*}
The existence of $P$ is guaranteed by the fact that $\Re(\spec(A)) > 0$. Then, differentiating the boundary function yields 
\begin{multline*}
    \dot{U} = x^\top P(I_d \otimes A + A \otimes I_d)+(I_d \otimes A + A \otimes I_d)^{\top}Px \\ + 2 x^\top P \text{vec}(\mathcal{N}(\Sigma))=\|x\|^2+2 x^\top P \text{vec}(\mathcal{N}(\Sigma)),
\end{multline*}
where $\mathcal{N}(\Sigma):=V\Sigma B \Sigma + \Sigma B^\top \Sigma V^\top$ is the quadratic nonlinearity. Therefore, there exists $\varepsilon>0$ such that when $U(\Sigma)=\varepsilon$, we have $\dot{U}=\langle \nabla U,\mathrm{vec}(\dot{\Sigma})\rangle>0$. Summarizing the above arguments, we construct the set $\mathcal{K}\coloneqq \{\Sigma\succ0\:|\:U(\Sigma)\geq \varepsilon, \tr(\Sigma)\leq  R\}$. Note that $\mathcal{K}$ is a compact convex set in the set of $d\times d$ positive definite matrices (and hence in $\R^{d(d+1)/ {2}}$) with nonempty interior. Moreover, on the boundary of $\mathcal{K}$, the $\Sigma$ vector field $\mathcal{V}(\Sigma)\coloneqq A\Sigma + \Sigma A^\top + V\Sigma B\Sigma + \Sigma B^\top \Sigma V^\top$ is transversal to the boundary $\partial\mathcal{K}$ and points towards the interior of $\mathcal{K}$ at every point on the boundary, due to the above calculation of the derivatives of $U(\Sigma)$ and $\tr(\Sigma)$. Therefore, by the Poincar\'{e}-Hopf theorem, we have 
\begin{align}\label{6.00}
    \sum_{\{\Sigma^* \in \mathcal{K}\:|\: \mathcal{V}(\Sigma^*) = 0\}} \mathrm{ind}(\Sigma^*) =(-1)^\frac{d(d+1)}{2}\chi(\mathcal{K})= (-1)^\frac{d(d+1)}{2},
\end{align}
where the last equality is due to the fact that any compact, convex subset of $\mathbb{R}^N$ ($N>0$) with non-empty interior is homeomorphic to the closed unit ball $\mathbb{D}^N$, whose Euler characteristic $\chi$ is $1$. 

The non-zeroness of \eqref{6.00} implies that inside $\mathcal{K}$ there exists at least one equilibrium of the vector field $\mathcal{V}$, guaranteeing that \eqref{eq:Bernoulli-Algebraic} has at least one positive definite solution.

Let $\Sigma_{\infty}\succ0$ be one solution of \eqref{eq:Bernoulli-Algebraic}. We shall prove the local stability of $\Sigma_{\infty}$ under the condition that $\Sigma_\infty^{-1}V+V\Sigma_\infty^{-1}\prec0$, by showing that locally the flow of equation of $\Sigma$ \eqref{eq:gaussianTransformer} converges to the solution of \eqref{eq:Bernoulli-Algebraic}. Define $\Delta(t):=\Sigma(t)-\Sigma_{\infty}$, then
\begin{align}\label{6.4}
    \dot{\Delta}=M_\infty\Delta+\Delta M_\infty^{\top}+V\Delta B \Sigma_{\infty}+\Sigma_{\infty} B^{\top}\Delta V^{\top}+V\Delta B \Delta+\Delta B^{\top}\Delta V^{\top},
\end{align}
where $M_\infty := A + V\Sigma_\infty B$. Write $P := \Sigma_\infty^{-1} \succ 0$, and define the quadratic functional $\Phi(\Delta) := \frac12 \tr(\Delta P \Delta P) \equiv \frac{1}{2} \|P^{\frac{1}{2}} \Delta P^{\frac{1}{2}} \|_F^2$. It holds that
\begin{equation}\label{eq:norm-equivalence}
\frac12 \lambda_{\min}(P)^2 \|\Delta\|_F^2
\;\le\;
\Phi(\Delta)
\;\le\;
\frac12 \lambda_{\max}(P)^2 \|\Delta\|_F^2.
\end{equation}
We compute the derivative of $\Phi$ along solutions of \eqref{6.4}. Using the cyclic property of the trace, we obtain $\dot{\Phi}(\Delta) = \operatorname{tr}(\Delta P \dot{\Delta} P)$. Substituting \eqref{6.4}, we decompose $\dot{\Phi} = \mathrm{I} + \mathrm{II} + \mathrm{III}$, where
\begin{align*}
\mathrm{I}
&:= \tr \bigl(\Delta P (M_\infty \Delta + \Delta M_\infty^\top) P\bigr),\\
\mathrm{II}
&:= \tr \bigl(\Delta P (V\Delta B\Sigma_\infty + \Sigma_\infty B \Delta V) P\bigr),\\
\mathrm{III}
&:= \tr\bigl(\Delta P (V\Delta B\Delta + \Delta B\Delta V) P\bigr).
\end{align*}
Using cyclicity of the trace, we have $\mathrm{I} = \tr\bigl(\Delta P \Delta (P M_\infty + M_\infty^\top P)\bigr)$. Multiplying \eqref{eq:Bernoulli-Algebraic} on the left and right by $P$ yields $P M_\infty + M_\infty^\top P = 0$, hence $\mathrm{I} = 0$. Further, using $\Sigma_\infty P = P \Sigma_\infty = I$, we obtain
\begin{align*}
\mathrm{II}=\operatorname{tr}(\Delta P V \Delta B)+\operatorname{tr}(\Delta B \Delta V P)=
\operatorname{tr}\bigl(B \Delta (P V + V P) \Delta\bigr).
\end{align*}
By assumption, we have that $P V + V P \prec 0$. 
Set $S := -(P V + V P) \succ 0$. Then
$\mathrm{II}=
-\operatorname{tr}(B \Delta S \Delta)$.
Since $B \succ 0$ and $S \succ 0$, we have
\[
\operatorname{tr}(B \Delta S \Delta)
=
\|B^{\frac{1}{2}} \Delta S^{\frac{1}{2}}\|_F^2
\;\ge\;
\lambda_{\min}(B)\lambda_{\min}(S)\|\Delta\|_F^2.
\]
Therefore by \eqref{eq:norm-equivalence}, there exists $c>0$ such that
$\mathrm{II}
\le
- c\, \Phi(\Delta)$.
As for the cubic term $\mathrm{III}$, by \eqref{eq:norm-equivalence}, there exists $C_1>0$ such that
$|\mathrm{III}|
\le
C_1 \Phi(\Delta)^{\frac{3}{2}}$. Combining the above estimates, we obtain $\dot{\Phi}(\Delta)
\le
- c \Phi(\Delta) + C_1 \Phi(\Delta)^{\frac{3}{2}}$. Hence, there exists $r>0$ such that if $\Phi(\Delta) \le r$, then
$\dot{\Phi}(\Delta)
\le
-\frac{c}{2} \Phi(\Delta)$.
This implies exponential decay of $\Phi(\Delta(t))$ for all initial data sufficiently close to $\Sigma_\infty$. Using \eqref{eq:norm-equivalence} yields $\|\Sigma(t) - \Sigma_\infty\|_F
\le
C_2 e^{-\gamma t} \|\Sigma(0) - \Sigma_\infty\|_F
$ for some $C_2,\gamma>0$, proving this case. 

\paragraph{Case $\Re(\spec(A)) \leq 0$} We note that $\Sigma_\infty = 0$ is always a solution of \eqref{eq:Bernoulli-Algebraic}. Next, we show its global or local stability depending on the sign of $\lambda_{\max}(A+A^\top)$. Calculating the derivative of $\tr(\Sigma)$ yields
\begin{align}
    \tr(\dot{\Sigma})
    \le \lambda_{\max}(A+A^{\top})\tr(\Sigma)+ 2 \lambda_{\max}(V)\lambda_{\max}(B) 
    \tr(\Sigma^2). \label{eq:DiffSigma}
\end{align}
Hence if $\lambda_{\max}(A+A^\top)<0$ and $V\prec0$, $B\succ0$, then $\tr(\dot{\Sigma})
\le\lambda_{\max}(A+A^\top) \tr(\Sigma)$, and by Gr\"{o}nwall's Lemma, $\Sigma(t)$ converges to $0$ exponentially. If $\lambda_{\max}(A+A^\top)=0$, then \eqref{eq:DiffSigma} yields $\tr(\dot{\Sigma})\le 2 \lambda_{\max}(V)\lambda_{\max}(B) \tr(\Sigma^2)$, which is strictly negative as long as $\Sigma\neq0$, hence assuring that $\lim_{t\rightarrow\infty}\Sigma(t)=0$. 

If, instead, $\Re(\spec(A))  < 0$ and $\lambda_{\max}(A+A^\top)>0$, by first-order linearization of the dynamics of $\Sigma$ in \eqref{eq:gaussianTransformer}, $0$ is a locally stable equilibrium. Further, $\Sigma$ remains uniformly bounded over $t\in(0,+\infty)$ as we have the estimate
$$
    \tr(\dot{\Sigma})
    \le\lambda_{\max}(A+A^{\top})\tr(\Sigma)+\frac{2}{d} \lambda_{\max}(V)\lambda_{\max}(B) \tr(\Sigma)^2
    < 0
$$
when $\tr(\Sigma)> - \frac{\lambda_{\max}(A+A^\top)}{\frac{2}{d}\lambda_{\max}(V)\lambda_{\max}(B)}>0$. Hence completes the proof.
\end{proof}

\section*{Acknowledgments}
The authors thank the Speinshart Scientific Center for AI and SuperTech for its hospitality, where part of this research was conducted.

\bibliographystyle{siamplain}
\bibliography{references}
\end{document}